\documentclass[lettersize,journal]{IEEEtran}
\usepackage{mathrsfs}
\usepackage{amsmath}
\usepackage{amsthm}
\usepackage{tabularx}
\usepackage{booktabs}
\usepackage{amssymb}
\usepackage{graphicx}
\usepackage{subfigure}
\usepackage{times}
\usepackage{latexsym}
\usepackage{subfigure}
\usepackage{cite}
\usepackage{balance}
\usepackage{multirow}
\usepackage{algorithmic}
\usepackage{color}
\usepackage{float}
\usepackage{cases}
\usepackage{algorithm}
\usepackage{yhmath}
\usepackage{bm}
\usepackage{footnote}
\usepackage[justification=centering]{caption}
\newtheorem{lemma}{Lemma}
\newtheorem{theorem}{Theorem}

\theoremstyle{remark}
\newtheorem{remark}{Remark}

\begin{document}
\title{BEV-SGD: Best Effort Voting SGD against Byzantine Attacks for Analog Aggregation based Federated Learning Over the Air} 
\author{ Xin Fan$^{1}$, Yue Wang$^2$,~\IEEEmembership{Member,~IEEE}, Yan Huo$^{1}$,~\IEEEmembership{Senior Member,~IEEE}, and Zhi Tian$^2,~\IEEEmembership{Fellow,~IEEE}$\\
$^{1}$School of Electronics and Information Engineering, Beijing Jiaotong University, Beijing, China\\
$^2$Department of Electrical \& Computer Engineering, George Mason University, Fairfax, VA, USA
\\E-mails: \{yhuo,fanxin\}@bjtu.edu.cn, \{ywang56,ztian1\}@gmu.edu}
\maketitle

\begin{abstract}
As a promising distributed learning technology, analog aggregation based federated learning over the air (FLOA) provides high communication efficiency and privacy provisioning under the edge computing paradigm. When all edge devices (workers) simultaneously upload their local updates to the parameter server (PS) through commonly shared time-frequency resources, the PS obtains the averaged update only rather than the individual local ones. While such a concurrent transmission and aggregation scheme reduces the latency and communication costs, it unfortunately renders FLOA vulnerable to Byzantine attacks. Aiming at Byzantine-resilient FLOA, this paper starts from analyzing the channel inversion (CI) mechanism that is widely used for power control in FLOA.
Our theoretical analysis indicates that although CI can achieve good learning performance in the benign scenarios, it fails to work well with limited defensive capability against Byzantine attacks.
Then, we propose a novel scheme called the best effort voting (BEV) power control policy that is integrated with stochastic gradient descent (SGD). Our BEV-SGD enhances the robustness of FLOA to Byzantine attacks, by allowing all the workers to send their local updates at their maximum transmit power.
Under worst-case attacks, we derive the expected convergence rates of FLOA with CI and BEV power control policies, respectively. The rate comparison reveals that our BEV-SGD outperforms its counterpart with CI in terms of better convergence behavior, which is verified by experimental simulations.
\end{abstract}

\section{Introduction}
Edge intelligence has been recognized as a key enabler of various Internet-of-Things (IoT) services and applications in next-generation wireless systems\cite{zhu2020toward,mao2017survey}. Federated learning (FL) provides a promising paradigm for edge intelligence, by taking advantages of parallel computing at edge devices and privacy-aware access to rich distributed data\cite{konevcny2016federated,mcmahan2017communication,yang2019federated,chen2020joint}. 
To achieve communication-efficient FL, sparsification\cite{aji2017sparse, lin2017deep}, quantization\cite{liu2019decentralized,seide20141, alistarh2017qsgd} and infrequent uploading of local updates \cite{liu2019communication,8755802,chen2018lag,8646657,xu2020coke} are developed to reduce the amount of data to be digitally transmitted over wireless systems. However, the communication overhead and latency are still proportional to the number of local workers participated 
in FL over digital communication channels. To handle this issue, FL over the air (FLOA) is recently proposed as a new framework for distributed learning \cite{fan2021communication,fan2021joint,fan20211,zhu2019broadband,cao2019optimal,yang2020federated,zhu2020one,amiri2020machine, amiri2020federated1,amiri2019collaborative,sun2019energy}, which exploits the over-the-air computation (AirComp) principle\cite{nazer2007computation,gastpar2008uncoded} for ``one-shot” aggregation via local workers' simultaneous 
update transmission over 
the same time-frequency resources. 
Based on the inherent waveform superposition property of 
wireless multiple access channels (MAC),
AirComp 
allows to directly collect the gradient aggregation among local workers via concurrent transmission and computation
\cite{nazer2007computation,gastpar2008uncoded,abari2015airshare}, which exactly fits the need of FL for utilizing 
only an average of all distributed local gradients but not the individual values.

By virtue of its communication-efficient gradient aggregation, FLOA has attracted growing interest from multiple research communities to advance its development from the perspectives of communications, optimization and machine learning, such as power control\cite{cao2019optimal,fan2021joint,zhang2021gradient}, devices scheduling\cite{zhu2019broadband,fan2021joint,sun2019energy}, gradient compression\cite{zhu2020one,amiri2020machine, amiri2020federated1,fan2021communication,fan20211}, beamforming design\cite{yang2020federated,wang2021edge,amiri2019collaborative} and learning rate optimization\cite{xu2021learning}. For instance, a broadband analog aggregation scheme for 
power control and device scheduling in FLOA 
is proposed in \cite{zhu2019broadband}, where a set of tradeoffs between communications and learning are discussed. 
In \cite{fan2021joint}, convergence analysis is provided to quantify the impact of AirComp on FL and then joint optimization of communication and learning is proposed for optimal power scaling and device scheduling.
Considering energy-constrained local devices, an energy-aware device scheduling strategy is proposed in \cite{sun2019energy} to maximize the average number of workers scheduled for gradient update. 
For update compression, sparsification\cite{amiri2020machine,amiri2020federated1}, quantization\cite{zhu2020one} and compressive-sensing based methods \cite{fan2021communication,fan20211} are proposed to further improve communication efficiency. In multiple antennas scenarios, a joint design of device scheduling and beamforming is 
presented in \cite{yang2020federated} to maximize the number of selected workers under a given mean square error (MSE) requirement. Since hyper-parameters can also affect learning performance, a learning rate optimization scheme is proposed for multi-antenna systems to further improve the MSE performance and the testing accuracy\cite{xu2021learning}.

Beside its superior communication efficiency over conventional FL, FLOA also enhances the data privacy thanks to 
its inherent unaccessibility to individual local gradients, which thus prevent 
potential model inversion attacks, e.g., deep leakage from gradients\cite{zhu2020deep}. 
While FLOA 
closes the doors to deep leakage from gradients, it leaves the windows open for adversaries to perform Byzantine attacks as well. 
In fact, even a single Byzantine fault may destroy FL. Byzantine-robust aggregation has been well studied for vanilla FL\cite{yang2020adversary,yin2018byzantine,dong2019secure,damaskinos2019aggregathor}, most of which uses a screening method, such as geometric median\cite{minsker2015geometric,wu2020federated,chen2017distributed,huang2021byzantine}, coordinate-wise median\cite{yin2018byzantine}, coordinate-wise trimmed mean\cite{yin2018byzantine}, Krum/Multi-Krum\cite{blanchard2017machine}, Bulyan\cite{guerraoui2018hidden,el2019fast}, Zeno/Zeno++\cite{xie2018zeno,xie2019zeno++} and so on\cite{yang2020adversary}. The 
basic idea of these screening methods is to 
exclude outliers while aggregating the rest of local gradients. All of them 
hinge on the knowledge on the individual values of local gradients, which is however not accessible in FLOA due to the analog superposition of all local gradients over the air. Thus, existing Byzantine-robust methods designed for vanilla FL cannot be applied to FLOA, which motivates us to design a new Byzantine-resilient approach customized for FLOA. 

To the best of our knowledge, there is no literature so far on the study of Byzantine attacks to the over-the-air transmissions, nor is there any design of counter-attack measures for FLOA. In this work, we aim to
deeply understand how Byzantine attacks affect FLOA 
and then provide the corresponding defense strategy. Our main contributions are three-fold.
\begin{itemize}
\item Given the fact 
that most prior works on FLOA 
adopt channel inversion (CI) power control (or its variants)\cite{zhu2019broadband,fan2021joint,sun2019energy,zhu2020one,amiri2020machine, fan2021communication,fan20211,xia2020fast,xu2021learning,liu2020privacy}, we first theoretically prove that the CI methods under fading channels can achieve performance approximating
that of the ideal error-free case, 
which explains why it is widely used to overcome the transmission errors in FL.
Meanwhile, our analysis reveals that the defensive capacity of CI is very limited against Byzantine attacks. 
Thus, we propose a new robust transmission policy to counter Byzantine attacks, named the best effort voting (BEV) power control policy, where local workers transmit their local gradients with their maximum power.
\item To study the impact of Byzantine attacks to FLOA, we derive the transmission policy of intelligent Byzantine attackers, including the falsified gradients and transmit power, that can maximally deter the convergence of FLOA.
As this is the strongest attack, it is meaningful to assess its impact on FLOA under various transmission policies, which in turn serves to illuminate the respective robustness level of these policies.
\item To demonstrate the effectiveness of our proposed BEV method compared with
the popular CI scheme under the strongest attacks, 
we provide the convergence analysis for both our BEV and the existing CI. 
Our theoretical results prove that BEV outperforms CI in terms of better convergence behavior under the strongest Byzantine attacks. 
\end{itemize}

We also test the proposed method 
on image classification problems using the MNIST dataset. Simulation results show that the learning performance of BEV is slightly worse 
than that of CI when there are no Byzantine attacks, while BEV significantly outperforms 
CI in terms of the robustness to against Byzantine attacks. Thus, our theoretical analysis and simulation results suggest 
that BEV is preferred over CI in practical applications that are subject to Byzantine attacks.

The rest of this paper is organized as follows. The system model for FLOA is presented in Section II, where we provide two power control policies i.e., CI and BEV. The closed-form expressions of their expected
convergence rate are derived to compare the performance of different power control policies in Section III, where we also delineate the strongest 
attack case for a Byzantine attacker. Simulation results are provided in Section IV, followed by conclusions in Section V. 

\section{System Model}\label{sec:Model}
\subsection{Federated Learning Model}
Consider a distributed computation model with one parameter server (PS) and $U$ local workers. Each local worker stores $K$ data points, which are independent and identically distributed (i.i.d.) samples drawn from a large dataset $\mathcal{D}$\cite{yin2018byzantine,dong2019secure,damaskinos2019aggregathor}. The Byzantine-resilient issue for the non-i.i.d. case is more involved, which is left for future work. Denote $(\mathbf{x}_{i,k},\mathbf{y}_{i,k})$ as the $k$-th data of the $i$-th local worker. Let $f(\mathbf{w};\mathbf{x}_{i,k},\mathbf{y}_{i,k})$ denote the loss function associated with each data point $(\mathbf{x}_{i,k},\mathbf{y}_{i,k})$, where $\mathbf{w}=[w^1, \ldots, w^D]$ of size $D$ consists of the model parameters.
 The corresponding population loss function is denoted as $F(\mathbf{w}):=\mathbb{E}_{\mathcal{D}}[f(\mathbf{w};\mathbf{x}_{i,k},\mathbf{y}_{i,k})]$. The PS and local workers collaboratively learn the model parameter vector $\mathbf{w}$ by minimizing
\begin{align}\label{eq:lossfopt}
 \textbf{P1:} \quad \mathbf{w}^*= \arg \min_{\mathbf{w}}& \quad F(\mathbf{w}).
\end{align}

The minimization of $F(\mathbf{w})$ is typically carried out through stochastic gradient descent (SGD) algorithm. At the PS, the model parameter $\mathbf{w}_t$ at the $t$ iteration is updated as
\begin{align}\label{eq:localupdate0}
  \text{(Model updating)}\quad \mathbf{w}_{t}&=\mathbf{w}_{t-1}-\alpha \frac{\sum_{i=1}^{U}\mathbf{g}_{i,t}}{U},
\end{align}
where $\alpha$ is the learning rate and $\mathbf{g}_{i,t}=\nabla f(\mathbf{w}_{t-1};\mathbf{x}_{i,k},\mathbf{y}_{i,k})$ is the local gradient computed at the $i$-th local worker using its randomly selected the data sample, say the $k$-th sample. Some communication and aggregation scheme needs to be in place in order for the PS to acquire the sum of local gradients in \eqref{eq:localupdate0} from local workers.

Assume that $N$ out of $U$ local workers are Byzantine attackers, and the remaining $M=U-N$ local workers are normal. 
However, the Byzantine attackers do not need to follow 
this protocol and can send arbitrary messages to the PS. 
Even worse, these attackers may have complete knowledge of the learning system and algorithms, and can collude with each other.
Further, the communications between the PS and local workers inevitably introduce channel noise, while Byzantine attackers could also exploit this opportunity to disrupt FLOA.
Next, we will show that 
different predefined analog aggregation transmission 
protocols result in 
different performance of FLOA in the presence of Byzantine attacks. 

\subsection{Analog Aggregation Transmission Model}

In FLOA, to exploit over-the-air computation for low-latency gradient aggregation, local gradients are amplitude-modulated for analog transmission and simultaneously transmitted from local workers to the PS through the same multi-access channel. 
Assume that symbol-level synchronization is achieved among the local workers through a synchronization channel \cite{zhu2019broadband}.
To facilitate the power control design, the transmitted symbols, denoted by $\tilde{\mathbf{g}}_{i,t}=[\tilde{g}_{i,t}^1,...,\tilde{g}_{i,t}^d,...,\tilde{g}_{i,t}^D]$, are standardized such that they have zero mean and unit variance, i.e., $\mathbb{E}[(\tilde{g}_{i,t}^d)^2]=1, \forall i, t$. In this way, the power control policy can be designed at the PS without knowledge of the specific transmitted symbols. Note that the standardization factors are uniform for all local gradients and therefore can be inverted at the PS. 

Since the statistics of the gradients may change over iterations, the standardization is executed 
in all communication rounds. Specifically, at the beginning of each communication round, each local worker estimates its mean and variance of the locally learnt gradient, denoted by $\bar{g}_{i,t}=\frac{1}{D}\sum_{d=1}^{D} g_{i,t}^d$ and $\epsilon^2_{i,t}=\frac{1}{D}\sum_{d=1}^{D}(g_{i,t}^d-\bar{g}_{i,t})^2$, respectively.
Then the locally estimated mean and variance are transmitted to the PS for global gradient statistics estimation by averaging.
%
Given the received $\bar{g}_{i,t}$ and $\epsilon^2_{i,t}$, the PS averages all the local estimates to get the global estimates of the mean and variance of the gradient as $\bar{g}_{t}=\frac{1}{U}\sum_{i=1}^{U} \bar{g}_{i,t}$ and $\epsilon^2_{t}=\frac{1}{U}\sum_{i=1}^{U} \epsilon^2_{i,t}$.
Then the estimated $\bar{g}_{t}$ and $\epsilon^2_{t}$ are broadcast back to the local workers and used for the standardization.

  After receiving the standardization factors $\bar{g}_{t}$ and $\epsilon^2_{t}$, each local worker performs the transmit signal standardization as follows:
\begin{align}\label{eq:normalizedlocal}
 \tilde{\mathbf{g}}_{i,t}=\frac{\mathbf{g}_{i,t}-\bar{g}_{t}\mathbf{1}}{\epsilon_{t}},
\end{align}
where $\mathbf{1}$ is an all-one vector with dimension equal to that of $\mathbf{g}_{i,t}$. 

Considering only two symbols ($\bar{g}_{t}$ and $\epsilon^2_{t}$) transmitted in each communication round, the individual locally estimated mean and variance are collected at the PS one by one. We assume that such communications for standardization are noise-free without introducing errors.
Note that the Byzantine attackers know the designed standardization method, and they would send the true mean and variance of their local gradients to avoid exposing themselves during the standardization stage. Otherwise, the attackers may be easily detected and then filtered out by the PS, as the normal workers and Byzantine workers have i.i.d. datasets. 

After standardization, all local workers transmit their standardized local gradients $\tilde{\mathbf{g}}_{i,t}$ to the PS with certain transmit power $p_{i,t}$ (the design of power control on $p_{i,t}$ will be discussed later in this section).  
The transmission of each local worker is subject to the transmit power constraint:
\begin{align}
\mathbb{E}[\|p_{i,t}\tilde{\mathbf{g}}_{i,t}\|^2]&=\mathbb{E}[p_{i,t}^2\sum_{d=1}^D(\tilde{g}^d_{i,t})^2]=p_{i,t}^2\sum_{d=1}^D\mathbb{E}[(\tilde{g}^d_{i,t})^2]\nonumber\\
&=Dp_{i,t}^2\leq p_i^{\max}, \ \ \forall i.\label{eq:powerConstraint}
\end{align}
Thus the power constraint boils down to $p_{i,t}^2\leq\frac{p_i^{\max}}{D}$.

On the other hand, the Byzantine attackers can report any values of $\hat{\mathbf{g}}_{n,t}$ as their gradient updates to the PS so as to skew FL.  The transmit power $\hat{p}_{n,t}$ of the $n$-th Byzantine attackers satisfies
\begin{align}\label{eq:powerConstraintByzantine}
\mathbb{E}[\|\hat{p}_{n,t}\hat{\mathbf{g}}_{n,t}\|^2]\leq p_n^{\max}, \ \ \forall n.
\end{align}

Consider block fading channels, where the wireless channels remain unchanged within each iteration in FL but may change independently from one iteration to another.
We define the duration of one iteration as one time block, indexed by $t$. At the $t$-th iteration, the received signal at the PS is given by
\begin{align}\label{eq:receivedsignals}
  \mathbf{y}_{t}&=\sum_{m=1}^M p_{m,t}|h_{m,t}|\tilde{\mathbf{g}}_{m,t}+\sum_{n=1}^N \hat{p}_{n,t}|h_{n,t}|\hat{\mathbf{g}}_{n,t}+\mathbf{z}_{t},
 \end{align}
where the first, second, and third terms correspond to normal workers, attackers and noise, respectively. In particular, $|h_{i,t}|$ is the channel gain from the $i$-th worker to the PS at the $t$-th iteration and $\mathbf{z}_{t} \sim \mathcal{N}(0,z^2\mathbf{I})$ is additive white Gaussian noise (AWGN) that is independent of the gradient updates. The channels follow independent Rayleigh fading, i.e., $h_{i,t}\sim \mathcal{CN}(0,\sigma_{i}^2)$. In this work, we assume that the channels are perfectly known at local workers and the PS. With perfect channel state information (CSI), the channel phase offset is compensated at the local workers before they transmit their gradient updates.


After receiving the signals $\mathbf{y}_{t}$ in \eqref{eq:receivedsignals} from the local workers, the PS performs de-standardization to get the estimated aggregated gradient by inverting the standardization of \eqref{eq:normalizedlocal} as follows:
\begin{small}
\begin{align}
  \tilde{\mathbf{g}}_{t}=&\epsilon_{t}\mathbf{y}_{t} +\left(\sum_{i=1}^Up_{i,t}|h_{i,t}|\right)\bar{g}_{t}\mathbf{1} \nonumber\\=&\epsilon_{t}\left(\sum_{m=1}^M p_{m,t}|h_{m,t}|\tilde{\mathbf{g}}_{m,t}+\sum_{n=1}^N \hat{p}_{n,t}|h_{n,t}|\hat{\mathbf{g}}_{n,t}+\mathbf{z}_{t}\right)\nonumber\\ &+\left(\sum_{i=1}^Up_{i,t}|h_{i,t}|\right)\bar{g}_{t}\mathbf{1}
   \nonumber\\=&\epsilon_{t}\left(\sum_{m=1}^M p_{m,t}|h_{m,t}|\frac{\mathbf{g}_{m,t}-\bar{g}_{t}\mathbf{1}}{\epsilon_{t}}+\sum_{n=1}^N \hat{p}_{n,t}|h_{n,t}|\hat{\mathbf{g}}_{n,t}+\mathbf{z}_{t}\right)\nonumber\\ & +\left(\sum_{i=1}^Up_{i,t}|h_{i,t}|\right)\bar{g}_{t}\mathbf{1}
  \nonumber\\=&\sum_{m=1}^M p_{m,t}|h_{m,t}|\mathbf{g}_{m,t}+\epsilon_{t}\sum_{n=1}^N \hat{p}_{n,t}|h_{n,t}|\hat{\mathbf{g}}_{n,t} \nonumber\\ & +\left(\sum_{n=1}^Np_{n,t}|h_{n,t}|\right)\bar{g}_{t}\mathbf{1}+\epsilon_{t}\mathbf{z}_{t},\label{eq:de-normalization}
 \end{align}
 \end{small}where the first term corresponds to the aggregated gradients from normal local workers, the second plus the third terms denote the malignant contributions of Byzantine attackers to the gradient update, and the final term is from the noise.

By using the estimated aggregated gradient, the global model parameters are updated at the $t$-th iteration by
\begin{small}
\begin{align}\label{eq:modelupdate}
  \text{(updating with estimated gradients)}\quad \mathbf{w}_{t}&=\mathbf{w}_{t-1}-\alpha \tilde{\mathbf{g}}_{t}.
\end{align}
\end{small}
~Next, we discuss 
two transmit power allocation schemes for the design of $p_{i,t}$ that are adopted by 
normal local workers: the existing channel-inversion (CI) transmission \cite{zhu2019broadband,zhu2020one} and our proposed best effort voting (BEV) scheme.

\subsubsection{Channel-Inversion Transmission Scheme}

Given perfect known CSI, in the CI scheme\cite{zhu2019broadband,zhu2020one}, channels are inverted by power control so that gradient parameters transmitted by different local workers are received with identical amplitudes, which leads to 
amplitude alignment at the PS.
The transmit power of the $i$-th local worker is given by $p^2_{i,t}=\frac{b_t^2}{|h_{i,t}|^2}, \forall i,$
where $b_t^2=\min\{\frac{P_i^{\max}}{D}|h_{i,t}|^2, i=1,2,...,U\}$ is a scaling factor used to satisfy the power constraint in \eqref{eq:powerConstraint}.

It is evident that
\begin{align}\label{eq:powerchannel-inversion-bt}
\mathbb{E}[b_t^2]\geq P_0^{\max}\mathbb{E}[\min\{|h_{i,t}|^2, i=1,2,...,U\}],
\end{align}
where $P_0^{\max}=\min\{\frac{ P_i^{\max} }{D}, i=1,2,...,U\}$. Hence we can set $b_t^2=P_0^{\max}\mathbb{E}[\min\{|h_{i,t}|^2, i=1,2,...,U\}]$ for the power allocation. Since the channel coefficient is Rayleigh distributed $h_{i,t}\sim \mathcal{CN}(0,\sigma_{i}^2)$, $|h_{i,t}|^2$ follows the exponential distribution with mean $\frac{1}{\lambda_i}=2\sigma_{i}^2$. Thus, we have $\mathbb{E}[\min\{|h_{i,t}|^2, i=1,2,...,U\}]=\frac{1}{\sum_{i=1}^U \lambda_i}\doteq \lambda$.
As a result, for fulfilling the channel-inversion scheme in practice, the transmit power of the $i$-th local worker is set to
\begin{align}\label{eq:powerchannel-inversion}
p_{i,t}=\frac{b_0}{|h_{i,t}|},\quad \forall i,
\end{align}
where we set $b_0^2\doteq b_t^2=P_0^{\max}\lambda$.

\subsubsection{The Proposed Best Effort Voting Scheme}
 To counter intelligent Byzantine attackers, our idea is to let normal local workers try their best to combat the impact of potential Byzantine attacks so that FLOA converges to the right direction, which is therefore named as the best effort voting (BEV) scheme. In the BEV scheme, normal local workers transmit their local gradients by using their maximum transmit power which is independent to their CSI knowledge. The transmit power of the $i$-th local worker in BEV scheme is given by
\begin{align}\label{eq:powerBEV}
p_{i,t}=\sqrt{\frac{p_i^{\max}}{D}},\quad \forall i.
\end{align}

Different power allocation schemes have different resilience against Byzantine attackers, which we will discuss next.

\section{The Convergence Analysis}\label{sec:Convergence Analysis}
In this section, we compare the convergence performance of the aforementioned two power allocation schemes, CI and BEV. We first prove that there exists the strongest attack 
where a Byzantine attacker tries its best to 
prevent the convergence of FLOA. 
And then under such a circumstance, 
we derive 
the convergence rate of FLOA 
when applying the two transmission schemes, respectively.

\subsection{Assumptions}
To facilitate the convergence analysis, we make several standard assumptions on the loss function and the local gradient estimates.
Note that 
our theoretical derivations do not assume convexity on the loss function. Therefore, our methodology is also 
applicable to the popular learning models of deep neural networks (DNNs). 

\textbf{Assumption 1:} The loss function $F$ is Lipschitz continuous and smooth, that is,
\begin{align}
F(\mathbf{w}_{t})
\leq& F(\mathbf{w}_{t-1})+\mathbf{g}^T_t(\mathbf{w}_{t}-\mathbf{w}_{t-1})+\frac{L}{2}\|\mathbf{w}_{t}-\mathbf{w}_{t-1}\|^2,\label{Taylor}
\end{align}
where $L$ is a positive constant, referred to as the Lipschitz constant for the function $F(\cdot)$\cite{nesterov1998introductory}.

\textbf{Assumption 2:} The stochastic local gradient estimates are independent and unbiased estimates of the global gradient with the variance \cite{bernstein2018signsgd,zhu2020one}, i.e.,
\begin{align}\label{eq:unbiasedestimates}
&\mathbb{E}(\mathbf{g}_{i,t})=\mathbf{g}_t,\quad\forall i,t,\\
&\mathbb{E}(\|\mathbf{g}_{i,t}-\mathbf{g}_t\|^2)\leq \delta^2,\quad\forall i,t,\label{eq:bounded variance}
\end{align}
where we consider the standard SGD in this work. If 
the mini-batched SGD with a size $K_b$ is applied, then the variance is bounded by $\frac{\delta^2}{K_b}$.

\textbf{Assumption 3:} The standardization factors $\bar{g}_{t}$ and $\epsilon^2_{t}$ are unbiased estimates of the global gradient with the bounded variance as follows\cite{zhu2019broadband}
\begin{align}\label{eq:boundedGlobemean}
&\mathbb{E}[\bar{g}_{t}]=\frac{\sum_{d=1}^Dg_{t}^d}{D}, \quad \forall t,\\
&\epsilon_{t}\leq \epsilon, \quad \forall t.\label{eq:boundedGlobeVar}
\end{align}

The above assumptions allow tractable convergence analysis.

\subsection{The Strongest Byzantine Attacks}
While the Byzantine attackers may send arbitrary signals, there exists the strongest 
attack that a Byzantine attacker can achieve to prevent the convergence of FLOA. 
Intuitively, 
the Byzantine attackers would like to influence the global gradients at the PS along the opposite direction of that of normal local workers.
To this end, the Byzantine attackers will 
transmit $\hat{\mathbf{g}}_{n,t}=-\mathbf{g}_{n,t}$ to the PS with its maximum transmit power $\hat{p}_{n,t}$.
In particular, given the global model parameter $\mathbf{w}_{t-1}$, the Byzantine attackers compute its own 
gradient $\mathbf{g}_{n,t}$ by using their own local data.
In addition, the transmit power $\hat{p}_{n,t}$ satisfies the maximum power constraint, i.e., $\mathbb{E}[\|\hat{p}_{n,t}\hat{\mathbf{g}}_{n,t}\|^2]= p_n^{\max}$. This is the worst case that FLOA experiences in this work and we theoretically demonstrate in the following \textbf{Theorem \ref{Theorem:Worst-case}} that it is the strongest attack that a Byzantine attacker can impose to deter the convergence of FLOA.

\begin{theorem}\label{Theorem:Worst-case}
Employing SGD for the FL system deploying analog aggregation transmission in the presence of Byzantine attackers, the strongest attacks can be performed as
\begin{align}\label{eq:Worst-case1}
\hat{\mathbf{g}}_{n,t}&=-\mathbf{g}_{n,t},\\
\hat{p}_{n,t}&=\sqrt{\frac{p_n^{\max}}{(\bar{g}_{t}^2+\epsilon_{t}^2)D}}.\label{eq:Worst-case2}
\end{align}
\end{theorem}

\begin{proof}
The proof of \textbf{Theorem \ref{Theorem:Worst-case}} is provide in Appendix \ref{Appendix B}.
\end{proof}

Since the aforementioned strongest attack has been proved as the worst case that FLOA can experience, next we will 
evaluate the defense efficiency of different transmission schemes via convergence analysis.
We adopt the well known strategy of relating the norm of the gradient to the expected improvement to show the convergence for non-convex optimization\cite{wang2018cooperative,bernstein2018signsgd,zhu2020one}, i.e,
\begin{align}\label{eq:convergenceshow}
\min_{0,1,...,T}\mathbb{E}[\|\mathbf{g}_t\|^2]\leq\mathbb{E}\left[\sum_{t=1}^{T}\frac{1}{T}\|\mathbf{g}_t\|^2\right]\leq \mathcal{O}(\frac{1}{T^q}),
\end{align}
where $q>0$ is the order of the total number of the iterations $T$. As we can see, if \eqref{eq:convergenceshow} holds, the norm of the gradient is expected to converge to 0 as $T$ increases to infinity, which means that FL converges asymptotically. The convergence rate depends on the order value $q$, which is a key parameter to be assessed next.

\subsection{The Convergence of SGD with CI Transmission}
With 
CSI at each local worker, the CI power control can be 
performed as \eqref{eq:powerchannel-inversion}. The
resultant convergence rate of the CI transmission scheme under the strongest attacks 
is derived
as follows.
\begin{theorem}\label{Theorem:ConvergenceCI}
For a FLOA system with SGD-based model updating, CI-based power control for normal workers, and $N$ Byzantine attackers taking the strongest attacks as in \eqref{eq:Worst-case1}-\eqref{eq:Worst-case2}, the convergence rate is given by
\begin{align}\label{eq:ConvergenceCI}
\mathbb{E}[\sum_{t=1}^{T}\frac{1}{T}\|\mathbf{g}_t\|^2)]\leq&  \frac{1}{\sqrt{T}}\left(\frac{2L\Omega_{CI}}{\omega_{CI}^2\bar{\alpha}}(F(\mathbf{w}_{0}) -F(\mathbf{w}^*)) \right. \nonumber\\&\left.+\bar{\alpha}\left(\delta^2+\frac{1}{\Omega_{CI}}\epsilon^2 z^2\right)\right),
\end{align}
where
\begin{align}\label{TheoremTayloromegaCIs}
\omega_{CI}&=Mb_0-\sum_{n=1}^N \sqrt{\frac{\pi \sigma_n^2p_n^{\max}}{2D}},\\
\Omega_{CI}&=(U+N)\left(Ub_0^2+\sum_{n=1}^N \frac{2\sigma_n^2 p_n^{\max}}{D}\right),\label{TheoremTayloromegaCIL}
\end{align}
and $\bar{\alpha}=\frac{L\Omega_{CI}\sqrt{T}}{\omega_{CI}}\alpha$ is a positive constant satisfying $\bar{\alpha}<2\sqrt{T}$, and $b_0$ is initialized as in \eqref{eq:powerchannel-inversion}. The convergence is guaranteed if $\frac{\alpha^2L}{2}\Omega_{CI}-\alpha\omega_{CI}<0$, which imposes constraints on $\alpha$, $L$, $b_0$, $\sigma_n$, $p_n^{\max}$, $M$, $N$, $D$.
\end{theorem}
\begin{proof}
The proof of \textbf{Theorem \ref{Theorem:ConvergenceCI}} is provide in Appendix \ref{Appendix C}.
\end{proof}
\begin{remark}
For a small learning rate, the asymptotic convergence rate is dominated by $O(\frac{\Omega_{CI}}{\omega_{CI}^2\sqrt{T}})$. In addition, the convergence condition is given by $\frac{\alpha^2L}{2}\Omega_{CI}-\alpha\omega_{CI}<0$, the proof of which is also provided in Appendix \ref{Appendix C}. This condition imposes an upper bound on the learning rate in the form $\alpha < \frac{2\omega_{CI}}{L\Omega_{CI}}$. Further, when the learning rate is set to be small enough, $\alpha^2$ approaches 0, and the FL converges under a simplified condition of $\omega_{CI}>0$. From this convergence condition, we can see that even one Byzantine attacker can destroy the FLOA, if this attacker has a very large transmit power or its channel gain is very large, e.g., if $p_n^{\max}$ or $\sigma_n^2$ for any $n$ is very large, it is hard to ensure $\omega_{CI}>0$.
\end{remark}
\begin{remark}\label{remark:SP}
For a special case where all the local workers have the same maximum power (i.e., $p_i^{\max}=p^{\max}$, $\forall i$) and the independent and identically distributed channels (i.e., $\sigma_i=\sigma$, $\forall i$), we have the convergence condition $\omega_{CI}=(\frac{M}{\sqrt{U}}-\sqrt{\frac{N^2\pi}{4}})\sqrt{\frac{2p^{\max}\sigma^2}{D}}>0$. Therefore, we conclude that the number of attackers in this special case should be no more than $\frac{U}{1+\sqrt{\pi U}}$ to make the CI scheme defend against the Byzantine attack.
\end{remark}

When there are no Byzantine attackers, i.e., $N=0$, we have the following \textbf{Lemma \ref{Lemma:NOattacksCI}}.

\begin{lemma}\label{Lemma:NOattacksCI}
Employing SGD-based model updating for a FLOA system with the CI power control for normal local workers and no Byzantine attackers, the convergence rate is given by
\begin{align}\label{eq:LemmaConvergenceCI}
\mathbb{E}[\sum_{t=1}^{T}\frac{1}{T}\|\mathbf{g}_t\|^2)]\leq&  \frac{1}{\sqrt{T}}\left(\frac{2L}{\bar{\alpha}}(F(\mathbf{w}_{0})-F(\mathbf{w}^*)) \right. \nonumber\\&\left.+\bar{\alpha}\left(\delta^2+\frac{1}{U^2b_0^2}\epsilon^2 z^2\right)\right),
\end{align}
where $\alpha=\frac{1}{LUb_0\sqrt{T}}\bar{\alpha}$.
\end{lemma}
\begin{proof}
When $N=0$, we have $\omega^2_{CI}=\Omega_{CI}$. Then setting $\alpha=\frac{\omega_{CI}}{L\Omega_{CI}\sqrt{T}}\bar{\alpha}=\frac{1}{LUb_0\sqrt{T}}\bar{\alpha}$, substituting $\alpha$, $\omega_{CI}$ and $\Omega_{CI}$ into \eqref{eq:ConvergenceCI}, we complete the proof.
\end{proof}
\begin{remark}
As we can see from \eqref{eq:LemmaConvergenceCI}, in the case of CI power control without Byzantine attackers, we get the fastest asymptotic convergence rate as $O(\frac{1}{\sqrt{T}})$, which is the same as the error-free (EF) case where we do not consider the influence of wireless channels and noises.
\end{remark}

\subsection{The Convergence of SGD with BEV Transmission}
For our BEV transmission scheme under the strongest 
attacks, the resultant convergence rate is derived as following \textbf{Theorem \ref{Theorem:ConvergenceBEV}}.

\begin{theorem}\label{Theorem:ConvergenceBEV}
Employing SGD-based model updating for a FLOA system with the BEV power control for normal workers and $N$ Byzantine attackers taking the strongest attacks as in \eqref{eq:Worst-case1}-\eqref{eq:Worst-case2}, the convergence rate is given by
\begin{align}
\mathbb{E}[\sum_{t=1}^{T}\frac{1}{T}\|\mathbf{g}_t\|^2)]\leq&  \frac{1}{\sqrt{T}}\left(\frac{2L\Omega_{BEV}}{\bar{\alpha}\omega_{BEV}^2}(F(\mathbf{w}_{0}) -F(\mathbf{w}^*)) \right. \nonumber\\&\left.+\bar{\alpha}\left(\delta^2+\frac{1}{\Omega_{BEV}}\epsilon^2 z^2\right)\right),\label{eq:ConvergenceBEV}
\end{align}
where
\begin{align}\label{TheoremTayloromegaBEVs}
\omega_{BEV}&=\sum_{i=1}^M \sqrt{\frac{p_i^{\max}\pi}{2D}}\sigma_i-\sum_{n=1}^N \sqrt{\frac{p_n^{\max}\pi}{2D}}\sigma_n,\\
\Omega_{BEV}&=(U+N)\sum_{i=1}^U\frac{2\sigma_i^2 p_i^{\max}}{D},\label{TayloromegaBEVL}
\end{align}
and $\bar{\alpha}=\frac{L\Omega_{BEV}\sqrt{T}}{\omega_{BEV}}\alpha$ is a positive constant satisfying $\bar{\alpha}<2\sqrt{T}$. The convergence is guaranteed if $\frac{\alpha^2L}{2}\Omega_{BEV}-\alpha\omega_{BEV}<0$, which imposes constraints on $\alpha$, $L$, $\sigma_i$, $p_i^{\max}$, $M$, $N$, $D$.
\end{theorem}
\begin{proof}
The proof of \textbf{Theorem \ref{Theorem:ConvergenceBEV}} is provide in Appendix \ref{Appendix D}.
\end{proof}
\begin{remark}\label{remark:BEVcondition}
The proof of the convergence condition $\frac{\alpha^2L}{2}\Omega_{BEV}-\alpha\omega_{BEV}<0$ is provided in Appendix \ref{Appendix D}. This condition imposes an upper bound on the learning rate in the form $\alpha < \frac{2\omega_{BEV}}{L\Omega_{BEV}}$. Further, when the learning rate is set to be small enough, $\alpha^2$ approaches 0, and the FL converges under a simplified condition of $\omega_{BEV}>0$. If all the attackers and normal workers are isomorphic (the same case in \emph{Remark 2}), our BEV can defend Byzantine attacks when $N\leq \frac{U}{2}$. Since $\frac{U}{2}\geq \frac{U}{1+\sqrt{\pi U}}$, our BEV scheme can defend against a larger number of Byzantine attackers than that of CI.
\end{remark}
\begin{remark}
 For a small learning rate, if both the CI scheme and our BEV scheme can converge, the asymptotic convergence rate is dominated by $O(\frac{\Omega}{\omega^2\sqrt{T}})$. The comparison between $O(\frac{\Omega_{CI}}{\omega_{CI}^2\sqrt{T}})$ and $O(\frac{\Omega_{BEV}}{\omega_{BEV}^2\sqrt{T}})$ depends on the specific parameters.  For a large learning rate, if both the CI scheme and our BEV scheme can converge, the asymptotical convergence rate is dominated by $O(\frac{1}{\Omega\sqrt{T}})$. Since $\Omega_{BEV}> \Omega_{CI}$, the convergence rate of BEV scheme is faster than that of the CI scheme.
\end{remark}
\begin{remark}\label{Re:BEVworse}
When there are no Byzantine attackers, i.e., $N=0$, we have $\omega_{BEV}^2\leq \Omega_{BEV}$. Considering a small learning rate, the asymptotic convergence rate of BEV is dominated by $O(\frac{\Omega_{BEV}}{\omega_{BEV}^2\sqrt{T}})$, which is slower than both the 
CI scheme and the EF case.
\end{remark}

\section{Simulation Results}
To evaluate the resilience of our proposed BEV scheme against Byzantine attacks, we provide the simulation results for an image classification task. Unless specified otherwise, the simulation settings are given as follows.
The FLOA system has $U=10$ workers. The wireless channels between the workers and the PS are modeled as i.i.d. Rayleigh fading, by generating $h_{i,t}$'s from the complex Gaussian distribution $\mathcal{CN}(0, 1)$ for different $i$ and $t$. The average receive SNR at local workers is set to be $\frac{P^{\max}_i}{Dz^2}= 10$ dB \cite{zhu2020one}.

We consider the learning task of handwritten-digit identification using the well-known MNIST dataset\footnote{http://yann.lecun.com/exdb/mnist/} that consists of 10 classes ranging from digit ``0" to ``9". In the MNIST dataset, a total of 60000 labeled training data samples and 10000 test samples. In our experiments, we train a multilayer perceptron (MLP) with a 784-neuron input layer, a 64-neuron hidden layer, and a 10-neuron softmax output layer. We adopt rectified linear unit (ReLU) as the activation function, and cross entropy as the loss function. The total number of parameters in the MLP is $D=50890$. We randomly select $3000$ distinct training samples and distribute them to all local workers as their local datasets, i.e., $K_i=\bar{K}=3000$, for any $i \in [1, U]$.

We evaluate our BEV scheme under different attacks, including 1) without any attacks, 2) only one attacker who is far from the PS, hence a weak attacker, 3) only one attacker who is close to the PS, hence a strong attacker, and 4) randomly selected several attackers. We compare with two benchmarks: 1) the CI scheme and 2) the FLOA under the ideal error-free case (EF) where we do not consider the influence of wireless channels and noise.

\subsection{Performance without Attacks}
\begin{figure}[tb]
  \centering
  \subfigure[Training loss]{\includegraphics[width=0.40\textwidth]{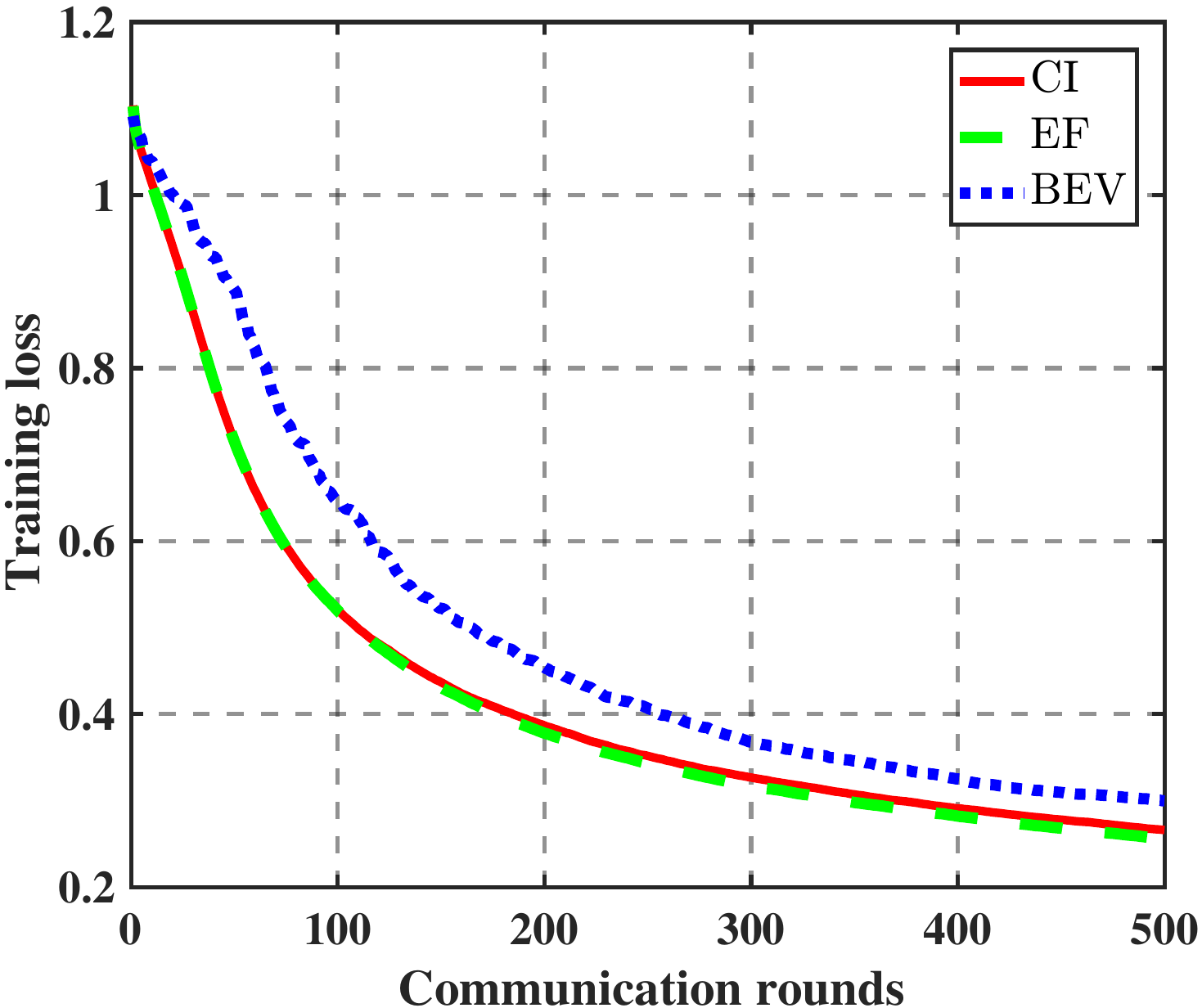}
  \label{fig:Lossatt} }
  \subfigure[Test accuracy]{\includegraphics[width=0.40\textwidth]{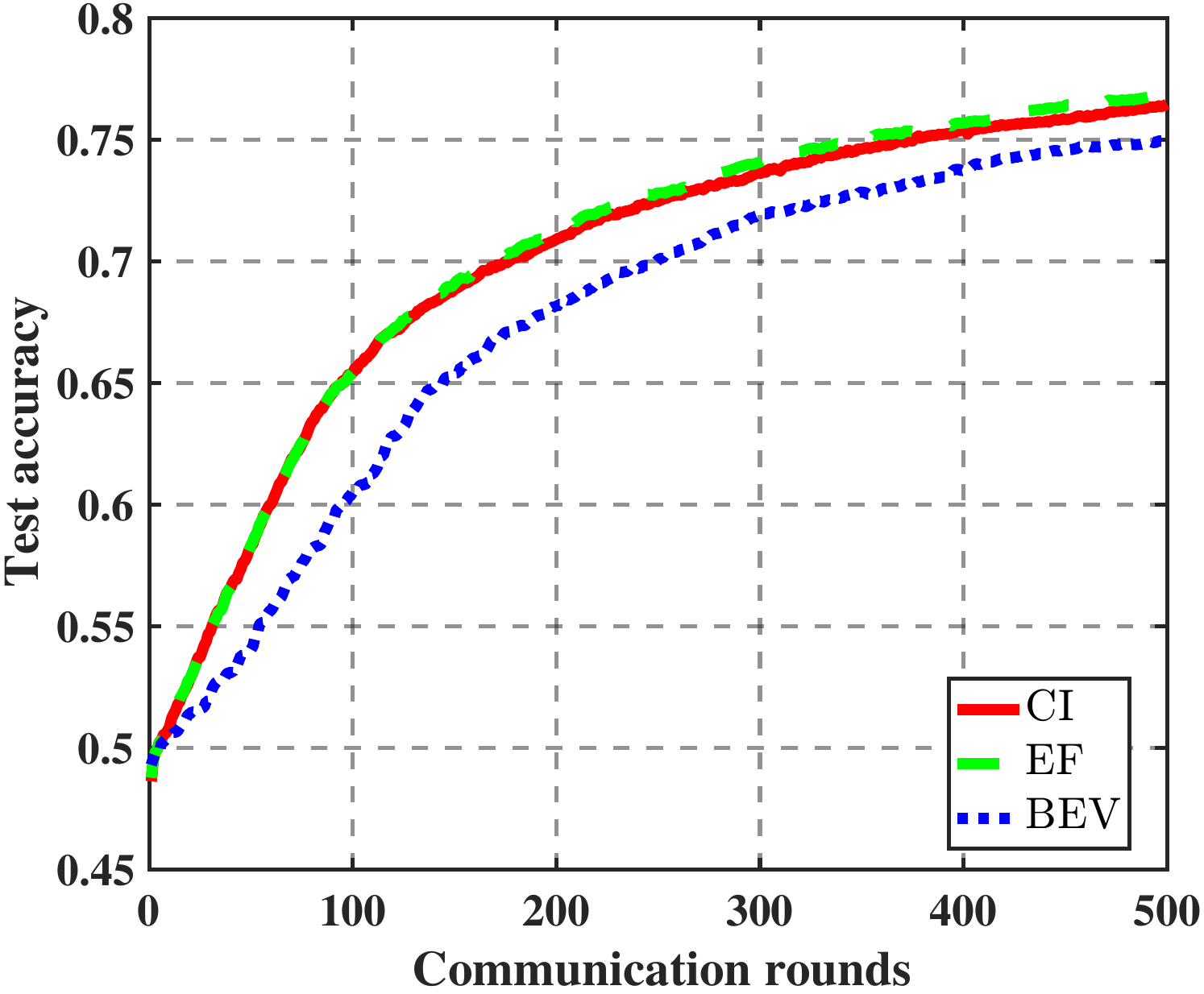}\label{fig:Accuracyatt}}
  \caption{The performance of BEV, CI and EF without Byzantine attacks.}\label{fig:withoutatt}
\end{figure}

The error-free case is set as the benchmark where the local gradients are perfectly aggregated at the PS, i.e., we set the channel $h_{i,t}=1$ and the AWGN $\mathbf{z}_t=0$. In Fig. \ref{fig:withoutatt}, we compare the performance of BEV with CI and EF without Byzantine attacks. Considering $\alpha<\frac{\omega}{L\Omega}$ in \emph{Remark 1} and \emph{Remark 4}, we set the learning rate $\alpha$ such as its scaled version is $\hat{\alpha}=\frac{\bar{\alpha}}{L\sqrt{T}}=\frac{\Omega}{\omega}\alpha=0.1$, where $\hat{\alpha}$ denotes the adjusting fact of $\alpha$. As we can see from Fig. \ref{fig:withoutatt}, the performance of CI is almost the same as EF. However, BEV experiences a 2\% performance loss compared to CI and EF. This results are in agreement with our theoretical analysis in Theorem 3, which has been discussed in Remark \ref{Re:BEVworse}. That is, CI converges a little faster than our BEV scheme, if and only if there exist no Byzantine attackers. However, practical learning applications of interest often operate in possible adversarial environments.

\subsection{Performance under a Single Attacker with Weak Channel Gain}

\begin{figure}[tb]
  \centering
  \subfigure[Training loss]{\includegraphics[width=0.40\textwidth]{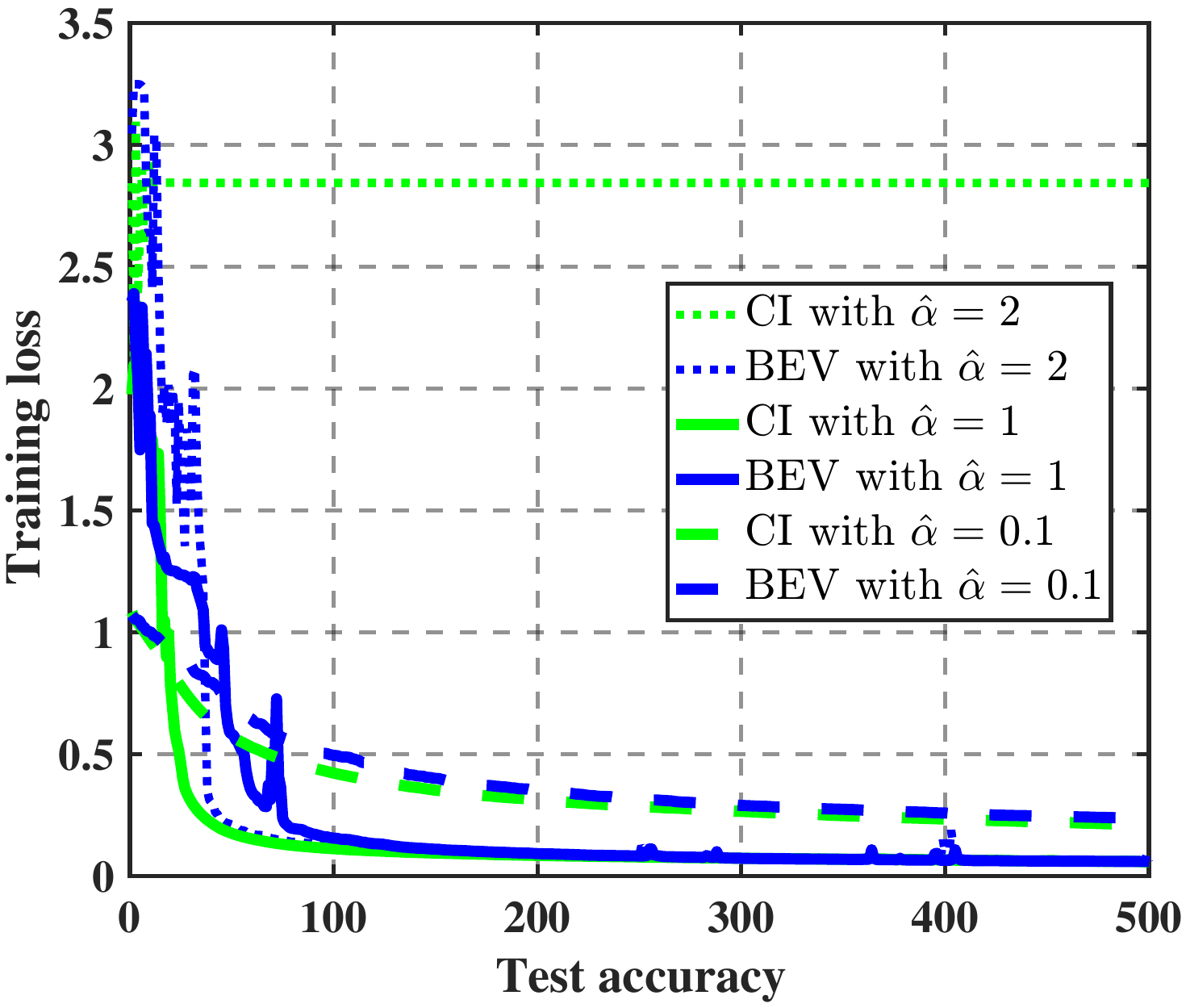}
  \label{fig:Lossattminh} }
  \subfigure[Test accuracy]{\includegraphics[width=0.40\textwidth]{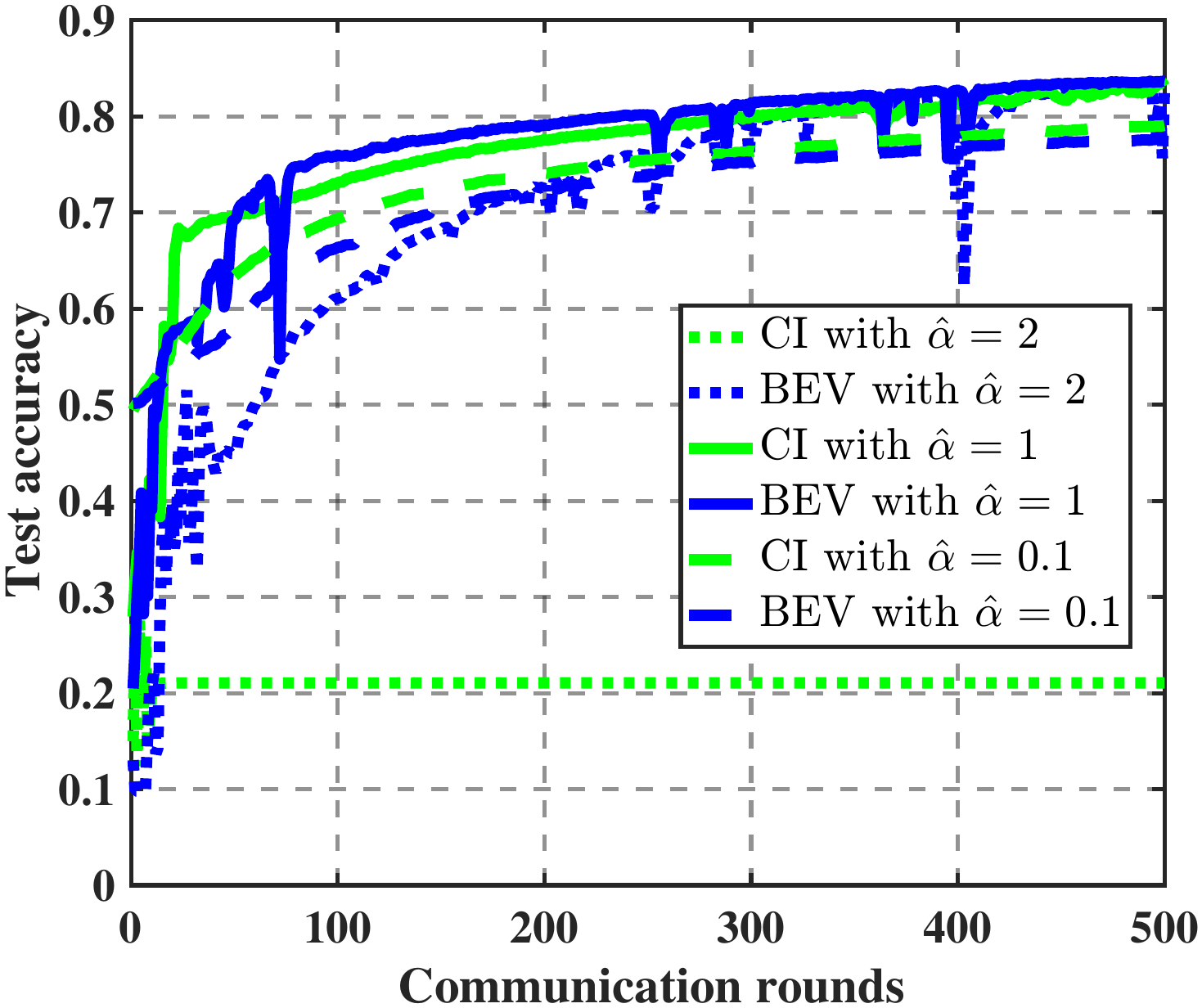}\label{fig:Accuracyattminh}}
  \caption{The performance of BEV and CI with a Byzantine attacker whose channel gain is the lowest.}\label{fig:withattminh}
\end{figure}

In Fig \ref{fig:withattminh}, we compare the performance of BEV with CI under a single Byzantine attack. Suppose that the attacker has the lowest channel gain among all local workers. It still adopts the strongest attack strategy to destroy FLOA.
Since the Byzantine attack to FLOA is relatively weak, both BEV and CI can converge, if a proper learning rate $\hat{\alpha}=\frac{\bar{\alpha}}{L\sqrt{T}}=\frac{\Omega}{\omega}\alpha$ is selected. On the other hand, when the learning rate is not properly chosen, e.g., when $\hat{\alpha}= 2$ in Fig. \ref{fig:withattminh}, BEV can converge but CI fails. When $\hat{\alpha}=1$, both BEV and CI can converge, but the convergence rate of BEV is faster than that of CI. This is because for a large learning rate, the asymptotic convergence rate is dominated by $O(\frac{1}{\Omega\sqrt{T}})$ and $\Omega_{BEV}> \Omega_{CI}$. When $\hat{\alpha}=0.1$, the performance of BEV is a little bit weaker in performance than CI. In practice, when the convergence can be guaranteed, we prefer a large learning rate to achieve a fast convergence rate. Under a large learning rate, e.g., $\hat{\alpha}=1$, our BEV works better than CI. 

\begin{figure}[tb]
  \centering
  \subfigure[Training loss]{\includegraphics[width=0.40\textwidth]{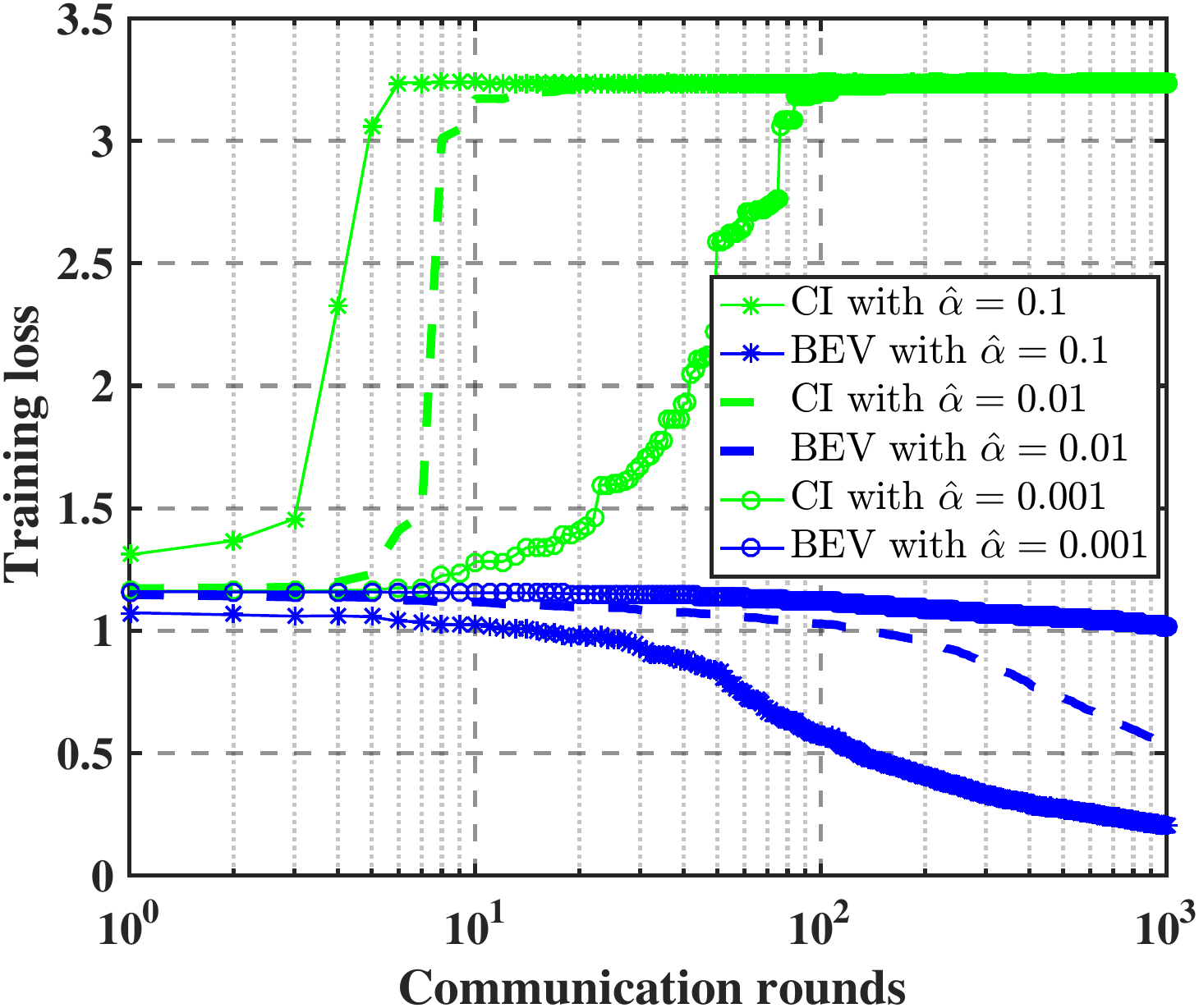}
  \label{fig:Lossattmaxh} }
  \subfigure[Test accuracy]{\includegraphics[width=0.40\textwidth]{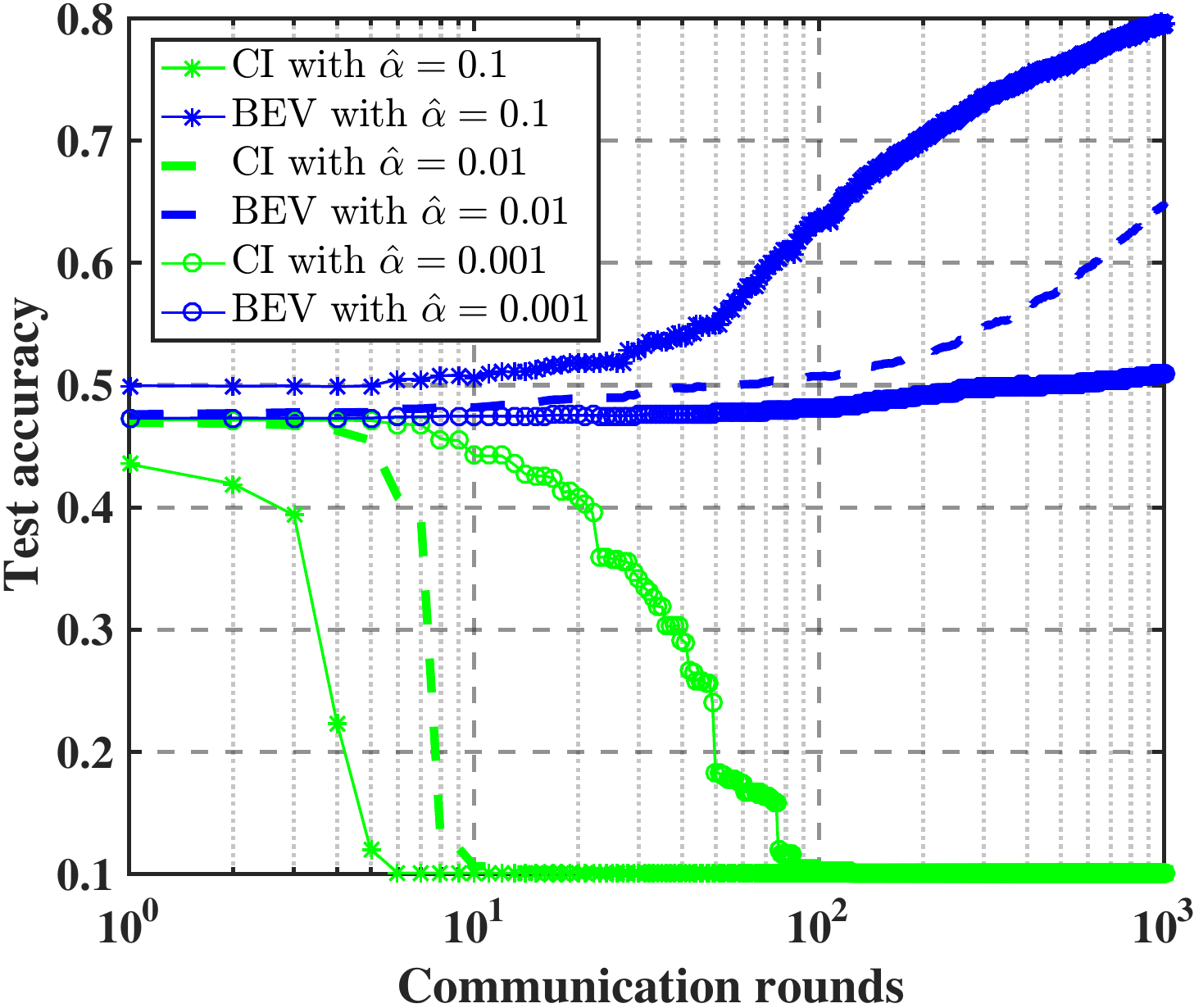}\label{fig:Accuracyattmaxh}}
  \caption{The performance of BEV and CI with a Byzantine attacker whose channel gain is the highest.}\label{fig:withattmaxh}
\end{figure}

\subsection{Performance under a Single Attacker with Large Channel Gain}
In Fig \ref{fig:withattmaxh}, we compare the performance of BEV with CI under a Byzantine attacker whose channel gain is the highest among all local workers. Thus, this is a strong attack. In this case of strong attacks, we compare the performance of BEV with CI under $\hat{\alpha}=\frac{\bar{\alpha}}{L\sqrt{T}}=\frac{\Omega}{\omega}\alpha$. Since the convergence condition $\omega_{CI}>0$ is hard to guarantee, it can be seen from Fig \ref{fig:withattmaxh} that CI cannot converge or coverage to a failure situation. As $\hat{\alpha}$ decreases, it is useful for CI to converge to the right direction, but it still cannot defend 
the attack after a few iterations. On the other hand, BEV can still converge, and hence is a better choice than CI in the presence of a strong attack. In addition, the convergence rate decreases as $\hat{\alpha}$ decreases. This implies that a larger learning rate is preferred under the condition of guaranteed convergence.

\begin{figure}[tb]
  \centering
  \subfigure[Training loss]{\includegraphics[width=0.40\textwidth]{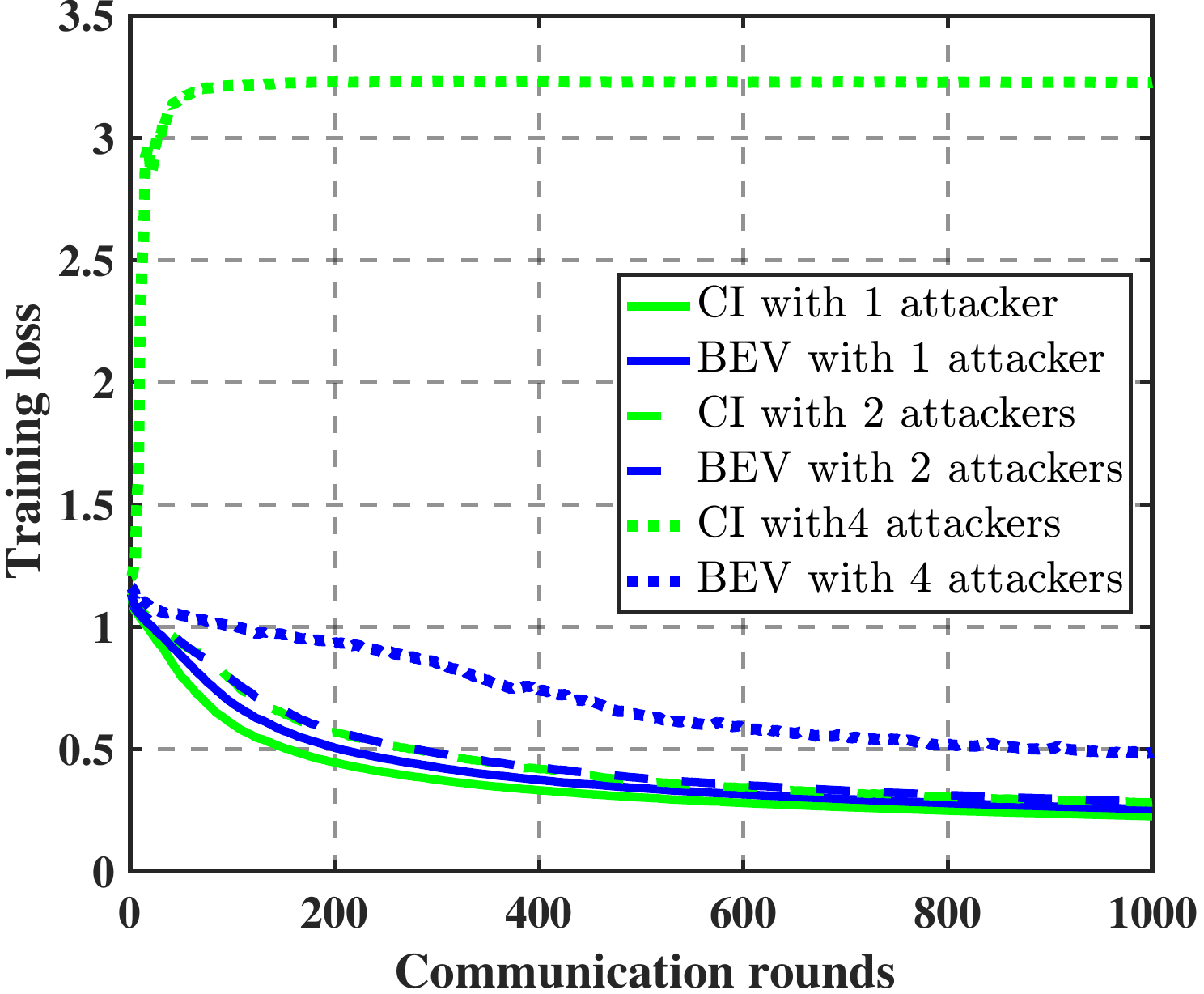}
  \label{fig:Lossattnum} }
  \subfigure[Test accuracy]{\includegraphics[width=0.40\textwidth]{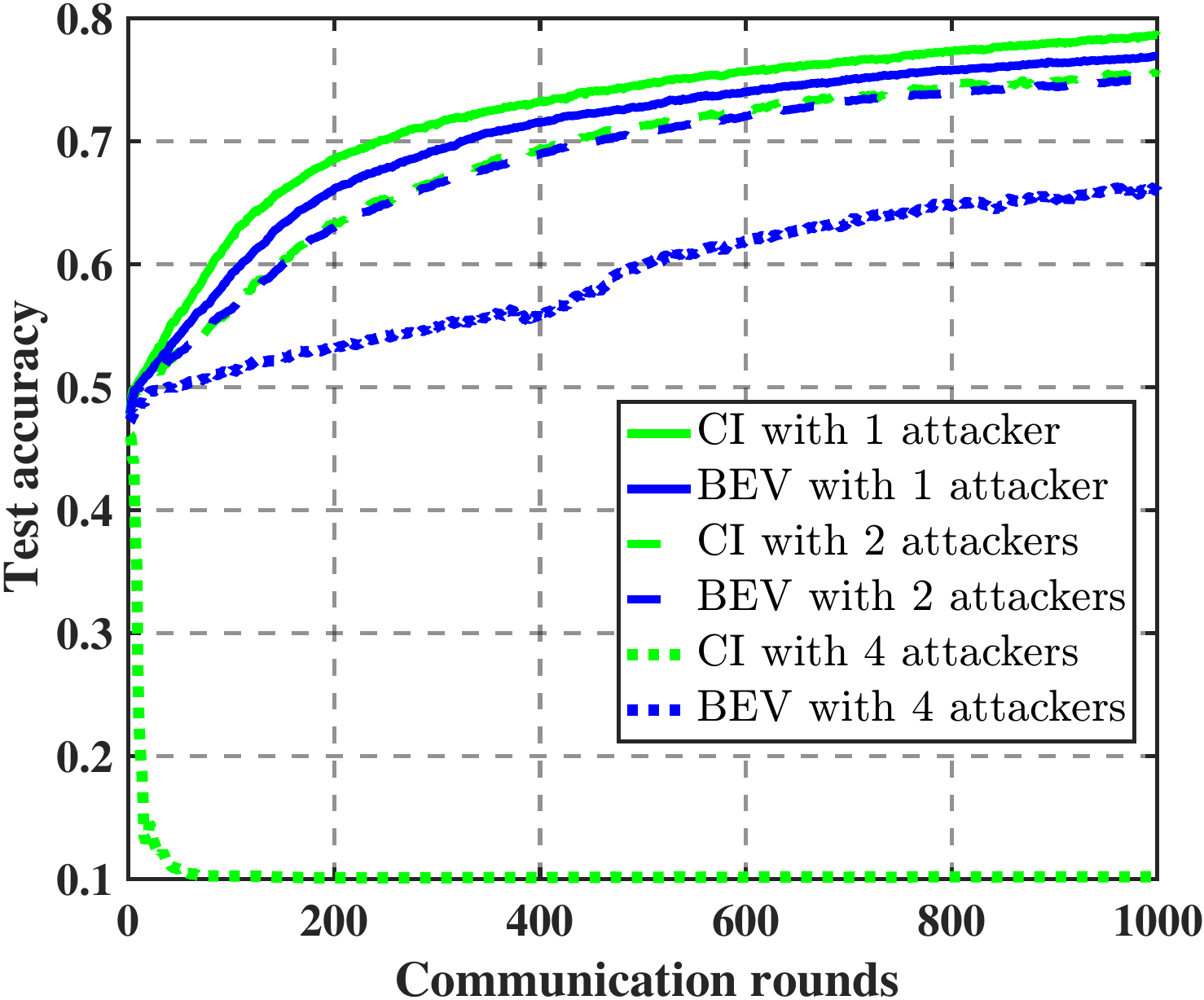}\label{fig:Accuracyattnum}}
  \caption{The performance of BEV and CI with the different number of Byzantine attackers.}\label{fig:withattnum}
\end{figure}

\subsection{Performance with Multiple Randomly Selected Attackers}
In Fig \ref{fig:withattnum}, 
we compare the performance of BEV with CI under the different number of Byzantine attackers. As we can see, when the number of Byzantine attackers is less than 4, both BEV and CI can converge, but the convergence rate decreases as the number of Byzantine attackers increases. When the number of Byzantine attackers is 4, i.e., $N>\frac{U}{1+\sqrt{\pi U}}$, CI can not converge to the correct direction, while BEV still converges in the correct direction but it converges at a slower rate. These results are 
consistent with our discussions in 
Remark \ref{remark:SP} and Remark \ref{remark:BEVcondition}.

\section{Conclusion}
This paper studies the robustness of FL over the air (FLOA) against Byzantine attacks. We provide theoretical analysis on convergence performance of different transmission schemes. Our analytical results reveal the strongest attack that Byzantine attackers can impose to deter FLOA from converging to the correct direction. Our convergence analyses, corroborated by simulation results, delineate the convergence behavior of the CI and BEV schemes under various adversarial environments. Specifically, in the absence of any Byzantine attacker, CI has the performance comparable to the ideal error-free case, while BEV has 2\% performance loss. In the weakest Byzantine attack, for a large learning rate, both CI and BEV can converge while BEV converges faster than CI. If there exists a strong Byzantine attacker, the convergence of CI cannot be guaranteed, but BEV can still converge. In practice, since it is impossible to determine the intensity of potential attacks, BEV is a better option to counter Byzantine attacks, because it performs well under various attack situations.

\section*{Acknowledgments}
This work was partly supported by the National Natural Science Foundation of China (Grants \#61871023 and \#61931001), Beijing Natural Science Foundation (Grant \#4202054), the National Science Foundation of the US (Grants \#1741338,
\#1939553, \#2003211,\#2128596 and \#2136202), and the Virginia Research Investment Fund (Commonwealth Cyber Initiative Grant \#223996).

\begin{appendices}
\section{Proof of \textbf{Theorem \ref{Theorem:Worst-case}}}\label{Appendix B}
 Given the estimates of the global gradient in \eqref{eq:de-normalization}, we have the update rule for model parameters as follows
\begin{align}
 \mathbf{w}_{t}=&\mathbf{w}_{t-1}-\alpha \tilde{\mathbf{g}}_{t}\nonumber\\
=&\mathbf{w}_{t-1}-\alpha\left(\sum_{m=1}^M p_{m,t}|h_{m,t}|\mathbf{g}_{m,t}+\epsilon_{t}\sum_{n=1}^N \hat{p}_{n,t}|h_{n,t}|\hat{\mathbf{g}}_{n,t}\right.\nonumber\\
  &\left.+\sum_{n=1}^Np_{n,t}|h_{n,t}|\bar{g}_{t}\mathbf{1}+\epsilon_{t}\mathbf{z}_{t}\right) . \label{eq:modelupdateunderWS}
\end{align}

Substituting \eqref{eq:modelupdateunderWS} to \eqref{Taylor}, we have
\begin{align}
F(\mathbf{w}_{t})
\leq& F(\mathbf{w}_{t-1})+\mathbf{g}^T_t(\mathbf{w}_{t}-\mathbf{w}_{t-1}) +\frac{L}{2}\|\mathbf{w}_{t}-\mathbf{w}_{t-1}\|^2
\nonumber \\
=& F(\mathbf{w}_{t-1})-\alpha\mathbf{g}^T_t\left(\sum_{m=1}^M p_{m,t}|h_{m,t}|\mathbf{g}_{m,t}\right.\nonumber \\&\left.+\epsilon_{t}\sum_{n=1}^N \hat{p}_{n,t}|h_{n,t}|\hat{\mathbf{g}}_{n,t} +\sum_{n=1}^Np_{n,t}|h_{n,t}|\bar{g}_{t}\mathbf{1}+\epsilon_{t}\mathbf{z}_{t}\right) \nonumber \\&+\frac{\alpha^2L}{2}\left\|\sum_{m=1}^M p_{m,t}|h_{m,t}|\mathbf{g}_{m,t}+\epsilon_{t}\sum_{n=1}^N \hat{p}_{n,t}|h_{n,t}|\hat{\mathbf{g}}_{n,t} \right.\nonumber \\&\left.+\sum_{n=1}^Np_{n,t}|h_{n,t}|\bar{g}_{t}\mathbf{1}+\epsilon_{t}\mathbf{z}_{t}\right\|^2.
\label{Taylor1WS}
\end{align}

Rewriting this inequality and taking the expectation, we have
\begin{align}\label{TaylorexpectationWS}
\mathbb{E}[F&(\mathbf{w}_{t})-F(\mathbf{w}_{t-1})]
\leq -\alpha\mathbf{g}^T_t\left(\sum_{m=1}^M p_{m,t}|h_{m,t}|\mathbf{g}_{t}\right.\nonumber \\&\left.+\sum_{n=1}^N \hat{p}_{n,t}|h_{n,t}|\hat{\mathbf{g}}_{n,t}+\sum_{n=1}^Np_{n,t}|h_{n,t}|\mathbb{E}[\bar{g}_t]\mathbf{1}\right) \nonumber \\&+\frac{\alpha^2L}{2}\mathbb{E}\left[\left\|\sum_{m=1}^M p_{m,t}|h_{m,t}|\mathbf{g}_{m,t}+\epsilon_{t}\sum_{n=1}^N \hat{p}_{n,t}|h_{n,t}|\hat{\mathbf{g}}_{n,t} \right.\right.\nonumber \\&\left.\left.+\sum_{n=1}^Np_{n,t}|h_{n,t}|\bar{g}_{t}\mathbf{1}+\epsilon_{t}\mathbf{z}_{t}\right\|^2\right].
\end{align}

Since $\mathbf{g}^T_t\mathbb{E}[\bar{g}_t]\mathbf{1}=\mathbf{g}^T_t\frac{\sum_{d=1}^{D}g_t^{d}}{D}\mathbf{1}=\frac{(\sum_{d=1}^{D}g_t^{d})^2}{D}\geq 0$, we have
\begin{align}\label{TaylorexpectationWS1}
\mathbb{E}[F(&\mathbf{w}_{t})-F(\mathbf{w}_{t-1})]
\leq \nonumber \\&-\alpha\left(\sum_{i=1}^M p_{i,t}|h_{i,t}|\|\mathbf{g}_{t}\|^2 +\sum_{n=1}^N \hat{p}_{n,t}|h_{n,t}|\mathbf{g}_{t}^T\hat{\mathbf{g}}_{n,t}\right)\nonumber \\&+\frac{\alpha^2L}{2}\mathbb{E}\left[\left\|\sum_{m=1}^M p_{m,t}|h_{m,t}|\mathbf{g}_{m,t}+\epsilon_{t}\sum_{n=1}^N \hat{p}_{n,t}|h_{n,t}|\hat{\mathbf{g}}_{n,t} \right.\right.\nonumber \\&\left.\left.+\sum_{n=1}^Np_{n,t}|h_{n,t}|\bar{g}_{t}\mathbf{1}+\epsilon_{t}\mathbf{z}_{t}\right\|^2\right].
\end{align}

If $\mathbb{E}(F(\mathbf{w}_{t})-F(\mathbf{w}_{t-1}))\leq 0$, the objective decreases monotonically, then FL converges in mean.
As we can see from \eqref{TaylorexpectationWS1}, if we set the learning rate to be small enough, then the second term on the right hand side of \eqref{TaylorexpectationWS1} diminishes, and convergence is ensured as long as
\begin{align}\label{conditionWS}
\sum_{i=1}^M p_{i,t}|h_{i,t}|\|\mathbf{g}_{t}\|^2 +\sum_{n=1}^N \hat{p}_{n,t}|h_{n,t}|\mathbf{g}_{t}^T\hat{\mathbf{g}}_{n,t}>0.
\end{align}

In order to break the convergence condition in \eqref{conditionWS}, the $N$ Byzantine attackers would seek to make  $\mathbf{g}_{t}^T\hat{\mathbf{g}}_{n,t}<0$ for any $n$. In fact, the best way for them is to send $\hat{\mathbf{g}}_{n,t}=-\mathbf{g}_{n,t}$ with their maximum power so as to make $\mathbb{E}[\mathbf{g}_{t}^T\hat{\mathbf{g}}_{n,t}]=-\|\mathbf{g}_{t}\|^2<0$.

Given the power constraint in \eqref{eq:powerConstraintByzantine}, we have
\begin{align}\label{eq:pc_Ws}
\mathbb{E}[\|\hat{p}_{n,t}\hat{\mathbf{g}}_{n,t}\|^2]&=\hat{p}_{n,t}^2\sum_{d=1}^D \mathbb{E}[(g_{n,t}^d)^2]\nonumber \\&=\hat{p}_{n,t}^2D(\epsilon_{t}^2+\bar{g}_{t}^2)\leq p_n^{\max}.
\end{align}

As a result, the Byzantine attackers are supposed to send $\hat{\mathbf{g}}_{n,t}=-\mathbf{g}_{n,t}$ with their maximum power $\hat{p}_{n,t}=\sqrt{\frac{p_n^{\max}}{D(\epsilon_{t}^2+\bar{g}_{t}^2)}}$.

\section{Proof of \textbf{Theorem \ref{Theorem:ConvergenceCI}}}\label{Appendix C}
Given the estimates of the global gradient in \eqref{eq:de-normalization}, the power allocation policy in \eqref{eq:powerchannel-inversion}, and the strongest 
attacks in \textbf{Theorem~\ref{Theorem:Worst-case}}, we have the update rule for model parameters as follows
\begin{align}
 \mathbf{w}_{t}=&\mathbf{w}_{t-1}-\alpha \tilde{\mathbf{g}}_{t}\nonumber\\
=&\mathbf{w}_{t-1}-\alpha\left(\sum_{m=1}^M p_{m,t}|h_{m,t}|\mathbf{g}_{m,t}+\epsilon_{t}\sum_{n=1}^N \hat{p}_{n,t}|h_{n,t}|\hat{\mathbf{g}}_{n,t}\right.\nonumber \\&\left.+\sum_{n=1}^Np_{n,t}|h_{n,t}|\bar{g}_{t}\mathbf{1}+\epsilon_{t}\mathbf{z}_{t}\right) \nonumber\\
=&\mathbf{w}_{t-1}-\alpha\left(\sum_{m=1}^M b_0\mathbf{g}_{m,t}\right.\nonumber \\&\left.-\epsilon_{t}\sum_{n=1}^N \sqrt{\frac{p_n^{\max}}{D(\epsilon_{t}^2+\bar{g}_{t}^2)}}|h_{n,t}|\mathbf{g}_{n,t}+\sum_{n=1}^Nb_0\bar{g}_{t}\mathbf{1}+\epsilon_{t}\mathbf{z}_{t}\right). \label{eq:modelupdateunderCI}
\end{align}

Substituting \eqref{eq:modelupdateunderCI} to \eqref{Taylor}, we get
\begin{align}\label{Taylor1}
F(\mathbf{w}_{t})
&\leq F(\mathbf{w}_{t-1})+\mathbf{g}^T_t(\mathbf{w}_{t}-\mathbf{w}_{t-1})+\frac{L}{2}\|\mathbf{w}_{t}-\mathbf{w}_{t-1}\|^2
\nonumber \\
&= F(\mathbf{w}_{t-1})-\alpha\mathbf{g}^T_t\left(\sum_{m=1}^M b_0\mathbf{g}_{m,t}\right.\nonumber \\&\left.-\epsilon_{t}\sum_{n=1}^N \sqrt{\frac{p_n^{\max}}{D(\epsilon_{t}^2+\bar{g}_{t}^2)}}|h_{n,t}|\mathbf{g}_{n,t}+\sum_{n=1}^Nb_0\bar{g}_{t}\mathbf{1}+\epsilon_{t}\mathbf{z}_{t}\right) \nonumber \\&+\frac{\alpha^2L}{2}\left\|\sum_{m=1}^M b_0\mathbf{g}_{m,t}-\epsilon_{t}\sum_{n=1}^N \sqrt{\frac{p_n^{\max}}{D(\epsilon_{t}^2+\bar{g}_{t}^2)}}|h_{n,t}|\mathbf{g}_{n,t}\right.\nonumber \\&\left.+\sum_{n=1}^Nb_0\bar{g}_{t}\mathbf{1}+\epsilon_{t}\mathbf{z}_{t}\right\|^2.
\end{align}

Rewriting this inequality and taking the expectation, we get
\begin{align}\label{Taylorexpectation}
\mathbb{E}[F(\mathbf{w}_{t})&-F(\mathbf{w}_{t-1})]
\leq -\alpha\mathbf{g}^T_t\left(\sum_{m=1}^M b_0\mathbf{g}_{t}\right.\nonumber \\&\left.-\epsilon_{t}\sum_{n=1}^N \sqrt{\frac{p_n^{\max}}{D(\epsilon_{t}^2 +\bar{g}_{t}^2)}}\mathbb{E}[|h_{n,t}|]\mathbf{g}_{t}+\sum_{n=1}^N b_0\mathbb{E}[\bar{g}_t]\mathbf{1}\right) \nonumber \\&+\frac{\alpha^2L}{2}\mathbb{E}\left[\left\|\sum_{m=1}^M b_0\mathbf{g}_{m,t}+\sum_{n=1}^Nb_0\bar{g}_{t}\mathbf{1}\right.\right.\nonumber \\&\left.\left.-\epsilon_{t}\sum_{n=1}^N \sqrt{\frac{p_n^{\max}}{D(\epsilon_{t}^2+\bar{g}_{t}^2)}}|h_{n,t}|\mathbf{g}_{n,t} +\epsilon_{t}\mathbf{z}_{t}\right\|^2\right],
\end{align}
where $\mathbb{E}[|h_{i,t}|]=\sigma_i\sqrt{\frac{\pi}{2}}$, because of the Rayleigh distributed $|h_{i,t}|$.

Since $\mathbf{g}^T_t\mathbb{E}[\bar{g}_t]\mathbf{1}=\mathbf{g}^T_t\frac{\sum_{d=1}^{D}g_t^{d}}{D}\mathbf{1}=\frac{(\sum_{d=1}^{D}g_t^{d})^2}{D}\geq 0$, we have
\begin{align}
\mathbb{E}&[F(\mathbf{w}_{t})-F(\mathbf{w}_{t-1})]
\leq -\alpha\left(Mb_0\right.\nonumber \\&\left.-\epsilon_{t}\sum_{n=1}^N \sqrt{\frac{p_n^{\max}}{D(\epsilon_{t}^2 +\bar{g}_{t}^2)}}\sigma_n\sqrt{\frac{\pi}{2}}\right)\|\mathbf{g}_t\|^2 \nonumber \\&+\frac{\alpha^2L}{2}\mathbb{E}\left[\left\|\sum_{m=1}^M b_0\mathbf{g}_{m,t}-\epsilon_{t}\sum_{n=1}^N \sqrt{\frac{p_n^{\max}}{D(\epsilon_{t}^2+\bar{g}_{t}^2)}}|h_{n,t}|\mathbf{g}_{n,t} \right.\right.\nonumber \\&\left.\left.+\sum_{n=1}^Nb_0\bar{g}_{t}\mathbf{1}+\epsilon_{t}\mathbf{z}_{t}\right\|^2\right]
\nonumber \\&\leq -\alpha\left(Mb_0-\sum_{n=1}^N \sqrt{\frac{p_n^{\max}}{D}}\sigma_n\sqrt{\frac{\pi}{2}}\right)\|\mathbf{g}_t\|^2 \nonumber \\&+\frac{\alpha^2L}{2}\mathbb{E}\left[\left\|\sum_{m=1}^M b_0\mathbf{g}_{m,t}-\epsilon_{t}\sum_{n=1}^N \sqrt{\frac{p_n^{\max}}{D(\epsilon_{t}^2+\bar{g}_{t}^2)}}|h_{n,t}|\mathbf{g}_{n,t} \right.\right.\nonumber \\&\left.\left.+\sum_{n=1}^Nb_0\bar{g}_{t}\mathbf{1}+\epsilon_{t}\mathbf{z}_{t}\right\|^2\right].
\label{Taylorexpectation2}
\end{align}


Using the triangle inequality of norms and Jensen’s inequality, we have
\begin{align}\label{Jensenexpectation1}
&\mathbb{E}\left[\left\|\sum_{m=1}^M b_0\mathbf{g}_{m,t}-\epsilon_{t}\sum_{n=1}^N \sqrt{\frac{p_n^{\max}}{D(\epsilon_{t}^2+\bar{g}_{t}^2)}}|h_{n,t}|\mathbf{g}_{n,t} \right.\right.\nonumber \\&\left.\left.\qquad+\sum_{n=1}^Nb_0\bar{g}_{t}\mathbf{1} +\epsilon_{t}\mathbf{z}_{t}\right\|^2\right]
\nonumber \\&=\mathbb{E}\left[\left\|\sum_{m=1}^M b_0\mathbf{g}_{m,t}-\epsilon_{t}\sum_{n=1}^N \sqrt{\frac{p_n^{\max}}{D(\epsilon_{t}^2+\bar{g}_{t}^2)}}|h_{n,t}|\mathbf{g}_{n,t} \right.\right.\nonumber \\&\left.\left.\qquad+\sum_{n=1}^Nb_0\bar{g}_{t}\mathbf{1}\right\|^2\right] +\mathbb{E}[\|\epsilon_{t}\mathbf{z}_{t}\|^2]
\nonumber \\
&\leq \mathbb{E}\left[\left(\sum_{m=1}^M \|b_0\mathbf{g}_{m,t}\|+\sum_{n=1}^N \left\|\epsilon_{t}\sqrt{\frac{p_n^{\max}}{D(\epsilon_{t}^2 +\bar{g}_{t}^2)}}|h_{n,t}|\mathbf{g}_{n,t}\right\| \right.\right.\nonumber \\&\left.\left.\qquad+\sum_{n=1}^N\|b_0\bar{g}_{t}\mathbf{1}\|\right)^2\right]+\epsilon^2 z^2
\nonumber \\&\leq \mathbb{E}\left[(U+N)\left(\sum_{m=1}^Mb_0^2\|\mathbf{g}_{m,t}\|^2 +\sum_{n=1}^Nb_0^2\|\bar{g}_{t}\mathbf{1}\|^2\right.\right.\nonumber \\&\left.\left.\qquad
+\sum_{n=1}^N \frac{\epsilon_{t}^2p_n^{\max}}{D(\epsilon_{t}^2 +\bar{g}_{t}^2)}|h_{n,t}|^2\|\mathbf{g}_{n,t}\|^2\right)\right] +\epsilon^2 z^2
\nonumber \\&= (U+N)\left(\sum_{m=1}^Mb_0^2\mathbb{E}[\|\mathbf{g}_{m,t}\|^2] +\sum_{n=1}^Nb_0^2D\left(\frac{\sum_{d=1}^{D}g_t^{d}}{D} \right)^2\right.\nonumber \\&\qquad\left. +\sum_{n=1}^N \frac{\epsilon_{t}^2p_n^{\max}}{D(\epsilon_{t}^2 +\bar{g}_{t}^2)}\mathbb{E}[|h_{n,t}|^2]\mathbb{E}[\|\mathbf{g}_{n,t}\|^2] \right)+\epsilon^2 z^2
\nonumber \\&\leq(U+N)\left(\sum_{m=1}^Mb_0^2(\|\mathbf{g}_{t}\|^2+\delta^2)+\sum_{n=1}^Nb_0^2\|\mathbf{g}_{t}\|^2\right.\nonumber \\&\left.\qquad+\sum_{n=1}^N \frac{\epsilon_{t}^2p_n^{\max}}{D(\epsilon_{t}^2 +\bar{g}_{t}^2)}2\sigma_n^2(\|\mathbf{g}_{t}\|^2+\delta^2) \right)+\epsilon^2 z^2
\nonumber \\&\leq(U+N)\left(\left(Ub_0^2+\sum_{n=1}^N \frac{2\sigma_n^2 p_n^{\max}}{D}\right)\|\mathbf{g}_{t}\|^2\right.\nonumber \\&\left.\qquad+\left(Mb_0^2+\sum_{n=1}^N \frac{2\sigma_n^2 p_n^{\max}}{D}\right)\delta^2\right)+\epsilon^2 z^2.
\end{align}

Substituting \eqref{Jensenexpectation1} to \eqref{Taylorexpectation2}, we get
\begin{align}\label{Taylorexpectation3}
&\mathbb{E}[F(\mathbf{w}_{t})-F(\mathbf{w}_{t-1})]
\leq -\alpha\left(Mb_0-\sum_{n=1}^N \sqrt{\frac{\pi \sigma_n^2p_n^{\max}}{2D}}\right)\|\mathbf{g}_t\|^2 \nonumber \\&+\frac{\alpha^2L}{2}\left ((U+N)\left(\left(Ub_0^2+\sum_{n=1}^N \frac{2\sigma_n^2 p_n^{\max}}{D}\right)\|\mathbf{g}_{t}\|^2
\right.\right.\nonumber \\&\left.\left.
+\left(Mb_0^2+\sum_{n=1}^N \frac{2\sigma_n^2 p_n^{\max}}{D}\right)\delta^2\right)+\epsilon^2 z^2\right)
\nonumber \\&\leq\left(\frac{\alpha^2L}{2}\Omega_{CI}-\alpha\omega_{CI}\right)\|\mathbf{g}_t\|^2
+\frac{\alpha^2L}{2}(\Omega_{CI}\delta^2+\epsilon^2 z^2),
\end{align}
where
\begin{align}\label{TayloromegaCIs}
\omega_{CI}&=Mb_0-\sum_{n=1}^N \sqrt{\frac{\pi \sigma_n^2p_n^{\max}}{2D}},\\
\Omega_{CI}&=(U+N)\left(Ub_0^2+\sum_{n=1}^N \frac{2\sigma_n^2 p_n^{\max}}{D}\right).\label{TayloromegaCIL}
\end{align}

If $\mathbb{E}(F(\mathbf{w}_{t})-F(\mathbf{w}_{t-1}))\leq 0$, the objective decreases monotonically, then FL converges in mean. Thus, to ensure the convergence, we have the following convergence condition
\begin{align}\label{condition}
\frac{\alpha^2L}{2}\Omega_{CI}-\alpha\omega_{CI}<0.
\end{align}

Now extend the expectation over randomness in the trajectory, and perform a telescoping sum over the $T$ iterations:
\begin{align}\label{sumexpectation}
F(\mathbf{w}_{0})-F(\mathbf{w}^*)&\geq F(\mathbf{w}_{0})-\mathbb{E}[F(\mathbf{w}_T)]
\nonumber \\
&
=\mathbb{E}\left[\sum_{t=1}^{T}(F(\mathbf{w}_{t-1})-F(\mathbf{w}_{t}))\right]
\nonumber \\
&
\geq\mathbb{E}\left[ \sum_{t=1}^{T}\left(\left(\alpha\omega_{CI}-\frac{\alpha^2L}{2}\Omega_{CI}\right)\|\mathbf{g}_t\|^2 \right.\right.\nonumber \\&\left.\left.-\frac{\alpha^2L}{2}(\Omega_{CI}\delta^2+\epsilon^2 z^2)\right)\right].%
\end{align}
We can rearrange this inequality to yield the rate:
\begin{align}\label{sumexpectation1}
\mathbb{E}&\left[\sum_{t=1}^{T}\left (\left(\alpha\omega_{CI} -\frac{\alpha^2L}{2}\Omega_{CI}\right)\|\mathbf{g}_t\|^2\right)\right]\nonumber \\&\leq F(\mathbf{w}_{0})-F(\mathbf{w}^*)+\frac{\alpha^2L}{2}T(\Omega_{CI}\delta^2+\epsilon^2 z^2).
\end{align}

If FL converges, the condition \eqref{condition} holds, yielding $\alpha\omega_{CI}-\frac{\alpha^2L}{2}\Omega_{CI}> 0$, and then we get
\begin{align}\label{sumexpectation2}
\mathbb{E}\left[\sum_{t=1}^{T}\frac{1}{T}\|\mathbf{g}_t\|^2\right]\leq&  \frac{1}{T(\alpha\omega_{CI}-\frac{\alpha^2L}{2}\Omega_{CI})}\left(F(\mathbf{w}_{0})-F(\mathbf{w}^*) \right.\nonumber \\&\left.+\frac{\alpha^2L}{2}T(\Omega_{CI}\delta^2+\epsilon^2 z^2)\right).
\end{align}

Let $\alpha=\frac{\omega_{CI}}{L\Omega_{CI}\sqrt{T}}\bar{\alpha}$, where $\bar{\alpha}< 2\sqrt{T}$ is a positive constant, and then we have
\begin{align}
&\mathbb{E}\left[\sum_{t=1}^{T}\frac{1}{T}\|\mathbf{g}_t\|^2\right]\leq  \frac{1}{T\left(\bar{\alpha}\frac{\omega_{CI}^2}{L\Omega_{CI}\sqrt{T}} -\frac{\bar{\alpha}^2\omega_{CI}^2}{2LT\Omega_{CI}}\right)}\left(F(\mathbf{w}_{0})\right.\nonumber \\&\left.\qquad-F(\mathbf{w}^*) +\frac{\bar{\alpha}^2\omega_{CI}^2}{2L\Omega_{CI}}\left(\delta^2+\frac{1}{\Omega_{CI}}\epsilon^2 z^2\right)\right)\nonumber\\
&=\frac{1}{T(\frac{\bar{\alpha}}{\sqrt{T}}-\frac{\bar{\alpha}^2}{2T})}\left(\frac{L\Omega_{CI}}{\omega_{CI}^2}(F(\mathbf{w}_{0})-F(\mathbf{w}^*)) \right.\nonumber \\&\left.\qquad+\frac{\bar{\alpha}^2}{2}\left(\delta^2+\frac{1}{\Omega_{CI}}\epsilon^2 z^2\right)\right)
\nonumber\\
&\leq\frac{1}{T\frac{\bar{\alpha}}{2\sqrt{T}}}\left(\frac{L\Omega_{CI}}{\omega_{CI}^2}(F(\mathbf{w}_{0})-F(\mathbf{w}^*)) \right.\nonumber \\&\left.\qquad+\frac{\bar{\alpha}^2}{2}\left(\delta^2+\frac{1}{\Omega_{CI}}\epsilon^2 z^2\right)\right)
\nonumber\\
&=\frac{1}{\sqrt{T}\bar{\alpha}}\left(\frac{2L\Omega_{CI}}{\omega_{CI}^2}(F(\mathbf{w}_{0})-F(\mathbf{w}^*)) \right.\nonumber \\&\left.\qquad+\bar{\alpha}^2\left(\delta^2+\frac{1}{\Omega_{CI}}\epsilon^2 z^2\right)\right).\label{sumexpectationCI3}
\end{align}


\section{Proof of \textbf{Theorem \ref{Theorem:ConvergenceBEV}}}\label{Appendix D}
Given the estimates of the global gradient in \eqref{eq:de-normalization}, the power allocation policy in \eqref{eq:powerBEV}, and the strongest 
attacks in \textbf{Theorem \ref{Theorem:Worst-case}}, we get the update rule for model parameters as follows
\begin{align}
 &\mathbf{w}_{t}=\mathbf{w}_{t-1}-\alpha \tilde{\mathbf{g}}_{t}\nonumber\\
  &=\mathbf{w}_{t-1}-\alpha\left (\sum_{m=1}^M p_{m,t}|h_{m,t}|\mathbf{g}_{m,t}+\epsilon_{t}\sum_{n=1}^N \hat{p}_{n,t}|h_{n,t}|\hat{\mathbf{g}}_{n,t}\right.\nonumber \\&\left.+\sum_{n=1}^Np_{n,t}|h_{n,t}|\bar{g}_{t}\mathbf{1} +\epsilon_{t}\mathbf{z}_{t}\right) \nonumber\\
  &=\mathbf{w}_{t-1}-\alpha\left(\sum_{m=1}^M \sqrt{\frac{p_m^{\max}}{D}}|h_{m,t}|\mathbf{g}_{m,t}+\epsilon_{t}\mathbf{z}_{t}\right. \nonumber\\&\left.-\epsilon_{t}\sum_{n=1}^N \sqrt{\frac{p_n^{\max}}{D(\epsilon_{t}^2 +\bar{g}_{t}^2)}}|h_{n,t}|\mathbf{g}_{n,t}+\sum_{n=1}^N\sqrt{\frac{p_n^{\max}}{D}}|h_{n,t}|\bar{g}_{t}\mathbf{1} \right). \label{eq:modelupdateunderBEV}
\end{align}

Substituting \eqref{eq:modelupdateunderBEV} to \eqref{Taylor}, we get
\begin{align}\label{TaylorBEV1}
&F(\mathbf{w}_{t})
\leq F(\mathbf{w}_{t-1})+\mathbf{g}^T_t(\mathbf{w}_{t}-\mathbf{w}_{t-1}) +\frac{L}{2}\|\mathbf{w}_{t}-\mathbf{w}_{t-1}\|^2
\nonumber \\
&= F(\mathbf{w}_{t-1})-\alpha\mathbf{g}^T_t\left(\sum_{m=1}^M \sqrt{\frac{p_m^{\max}}{D}}|h_{m,t}|\mathbf{g}_{m,t}\right.\nonumber \\&\left.-\epsilon_{t}\sum_{n=1}^N \sqrt{\frac{p_n^{\max}}{D(\epsilon_{t}^2+\bar{g}_{t}^2)}}|h_{n,t}|\mathbf{g}_{n,t}+\epsilon_{t}\mathbf{z}_{t}\right. \nonumber \\&\left.+\sum_{n=1}^N\sqrt{\frac{p_n^{\max}}{D}}|h_{n,t}|\bar{g}_{t}\mathbf{1} \right) +\frac{\alpha^2L}{2}\left\|\sum_{m=1}^M \sqrt{\frac{p_m^{\max}}{D}}|h_{m,t}|\mathbf{g}_{m,t}\right.\nonumber \\&\left.-\epsilon_{t}\sum_{n=1}^N \sqrt{\frac{p_n^{\max}}{D(\epsilon_{t}^2+\bar{g}_{t}^2)}}|h_{n,t}|\mathbf{g}_{n,t} \right.\nonumber \\&\left.+\sum_{n=1}^N\sqrt{\frac{p_n^{\max}}{D}}|h_{n,t}|\bar{g}_{t}\mathbf{1} +\epsilon_{t}\mathbf{z}_{t}\right\|^2.
\end{align}

Rearranging this inequality and taking the expectation, we get
\begin{align}
&\mathbb{E}[F(\mathbf{w}_{t})-F(\mathbf{w}_{t-1})]
\leq -\alpha\mathbf{g}^T_t\left(\sum_{m=1}^M \sqrt{\frac{p_m^{\max}}{D}}\mathbb{E}[|h_{m,t}|\mathbf{g}_{m,t}]\right.\nonumber \\&-\epsilon_{t}\sum_{n=1}^N \sqrt{\frac{p_n^{\max}}{D(\epsilon_{t}^2 +\bar{g}_{t}^2)}}\mathbb{E}[|h_{n,t}|\mathbf{g}_{n,t}]\nonumber \\&\left.+\sum_{n=1}^N\sqrt{\frac{p_n^{\max}}{D}}\mathbb{E}[|h_{n,t}|\bar{g}_{t}\mathbf{1}] +\mathbb{E}[\epsilon_{t}\mathbf{z}_{t}]\right) \nonumber \\&+\frac{\alpha^2L}{2}\mathbb{E}\left[\left\|\sum_{m=1}^M \sqrt{\frac{p_m^{\max}}{D}}|h_{m,t}|\mathbf{g}_{m,t}+\epsilon_{t}\mathbf{z}_{t}\right.\right.\nonumber \\&\left.\left.-\epsilon_{t}\sum_{n=1}^N \sqrt{\frac{p_n^{\max}}{D(\epsilon_{t}^2+\bar{g}_{t}^2)}}|h_{n,t}|\mathbf{g}_{n,t} +\sum_{n=1}^N\sqrt{\frac{p_n^{\max}}{D}}|h_{n,t}|\bar{g}_{t}\mathbf{1} \right\|^2\right]
\nonumber \\&\leq-\alpha\left(\sum_{i=1}^M \sqrt{\frac{p_i^{\max}}{D}}\sigma_i\sqrt{\frac{\pi}{2}}\right.\nonumber \\&\left.-\epsilon_{t}\sum_{n=1}^N \sqrt{\frac{p_n^{\max}}{D(\epsilon_{t}^2+\bar{g}_{t}^2)}}\sigma_n \sqrt{\frac{\pi}{2}}\right)\|\mathbf{g}_{t}\|^2 \nonumber \\&+\frac{\alpha^2L}{2}\mathbb{E}\left[\left\|\sum_{m=1}^M \sqrt{\frac{p_m^{\max}}{D}}|h_{m,t}|\mathbf{g}_{m,t} +\sum_{n=1}^N\sqrt{\frac{p_n^{\max}}{D}}|h_{n,t}|\bar{g}_{t}\mathbf{1}\right.\right. \nonumber \\&\left.\left.-\epsilon_{t}\sum_{n=1}^N \sqrt{\epsilon_{t}^2\frac{p_n^{\max}}{D(\epsilon_{t}^2 +\bar{g}_{t}^2)}}|h_{n,t}|\mathbf{g}_{n,t} +\epsilon_{t}\mathbf{z}_{t}\right\|^2\right]
\nonumber \\&\leq-\alpha\sqrt{\frac{\pi}{2}}\left(\sum_{i=1}^M \sqrt{\frac{p_i^{\max}}{D}}\sigma_i-\sum_{n=1}^N \sqrt{\frac{p_n^{\max}}{D}}\sigma_n\right)\|\mathbf{g}_{t}\|^2 \nonumber \\&+\frac{\alpha^2L}{2}\mathbb{E}\left[\left\|\sum_{m=1}^M \sqrt{\frac{p_m^{\max}}{D}}|h_{m,t}|\mathbf{g}_{m,t} +\sum_{n=1}^N\sqrt{\frac{p_n^{\max}}{D}}|h_{n,t}|\bar{g}_{t}\mathbf{1}\right.\right.\nonumber \\&\left.\left.-\epsilon_{t}\sum_{n=1}^N \sqrt{\frac{p_n^{\max}}{D(\epsilon_{t}^2+\bar{g}_{t}^2)}}|h_{n,t}|\mathbf{g}_{n,t} +\epsilon_{t}\mathbf{z}_{t}\right\|^2\right].\label{TaylorexpectationBEV}
\end{align}

Using the triangle inequality of norms and Jensen’s inequality, we have
\begin{align}\label{JensenexpectationBEV1}
&\mathbb{E}\left[\left\|\sum_{m=1}^M \sqrt{\frac{p_m^{\max}}{D}}|h_{m,t}|\mathbf{g}_{m,t} +\sum_{n=1}^N\sqrt{\frac{p_n^{\max}}{D}}|h_{n,t}|\bar{g}_{t}\mathbf{1} \right.\right.\nonumber \\&\left.\left.-\epsilon_{t}\sum_{n=1}^N \sqrt{\frac{p_n^{\max}}{D(\epsilon_{t}^2+\bar{g}_{t}^2)}}|h_{n,t}|\mathbf{g}_{n,t} +\epsilon_{t}\mathbf{z}_{t}\right\|^2\right]
\nonumber \\
&=
\mathbb{E}\left[\left\|\sum_{m=1}^M \sqrt{\frac{p_m^{\max}}{D}}|h_{m,t}|\mathbf{g}_{m,t} +\sum_{n=1}^N\sqrt{\frac{p_n^{\max}}{D}}|h_{n,t}|\bar{g}_{t}\mathbf{1}\right.\right.\nonumber \\&\left.\left.-\epsilon_{t}\sum_{n=1}^N \sqrt{\frac{p_n^{\max}}{D(\epsilon_{t}^2+\bar{g}_{t}^2)}}|h_{n,t}|\mathbf{g}_{n,t}\right\|^2\right] +\mathbb{E}[\|\epsilon_{t}\mathbf{z}_{t}\|^2]
\nonumber \\
&\leq
\mathbb{E}\left[\left(\sum_{m=1}^M \sqrt{\frac{p_m^{\max}}{D}}|h_{m,t}|\|\mathbf{g}_{m,t}\| +\sum_{n=1}^N\sqrt{\frac{p_n^{\max}}{D}}|h_{n,t}|\|\bar{g}_{t}\mathbf{1}\| \right.\right. \nonumber \\&\left.\left.+\epsilon_{t}\sum_{n=1}^N \sqrt{\frac{p_n^{\max}}{D(\epsilon_{t}^2+\bar{g}_{t}^2)}}|h_{n,t}|\|\mathbf{g}_{n,t}\|\right)^2\right] +\epsilon^2z^2
\nonumber \\
&\leq \mathbb{E}\left[(U+N)\left(\sum_{m=1}^M\frac{p_m^{\max}}{D}|h_{m,t}|^2\|\mathbf{g}_{m,t}\|^2 \right.\right. \nonumber \\&\left.\left.+\sum_{n=1}^N \frac{\epsilon_{t}^2p_n^{\max}}{D(\epsilon_{t}^2 +\bar{g}_{t}^2)}|h_{n,t}|^2\|\mathbf{g}_{n,t}\|^2\right.\right. \nonumber \\&\left.\left.+\sum_{n=1}^N\frac{p_n^{\max}}{D}|h_{n,t}|^2\|\mathbf{g}_{t}\|^2\right)\right]+\epsilon^2z^2
\nonumber \\&\leq (U+N)\left(\sum_{i=1}^U\frac{p_i^{\max}}{D}2\sigma_i^2\|\mathbf{g}_{t}\|^2 +\sum_{m=1}^M\frac{p_m^{\max}}{D}2\sigma_m^2\delta^2\right. \nonumber \\&\left.+\sum_{n=1}^N \frac{\epsilon_{t}^2p_n^{\max}}{D(\epsilon_{t}^2 +\bar{g}_{t}^2)}2\sigma_n^2\delta^2\right)+\epsilon^2z^2
\nonumber \\
&\leq (U+N)\left(\sum_{i=1}^U\frac{p_i^{\max}}{D}2\sigma_i^2\|\mathbf{g}_{t}\|^2 +\sum_{i=1}^U\frac{p_i^{\max}}{D}2\sigma_i^2\delta^2\right)+\epsilon^2z^2.
\end{align}

Substituting \eqref{JensenexpectationBEV1} to \eqref{TaylorexpectationBEV}, we get
\begin{align}\label{TaylorexpectationBEV3}
&\mathbb{E}[F(\mathbf{w}_{t})-F(\mathbf{w}_{t-1})]\nonumber \\&
\leq -\alpha\sqrt{\frac{\pi}{2}}\left(\sum_{i=1}^M \sqrt{\frac{p_i^{\max}}{D}}\sigma_i-\sum_{n=1}^N \sqrt{\frac{p_n^{\max}}{D}}\sigma_n\right)\|\mathbf{g}_t\|^2 \nonumber \\&+\frac{\alpha^2L}{2}\left((U+N)\left(\sum_{i=1}^U\frac{p_i^{\max}}{D}2\sigma_i^2 \|\mathbf{g}_{t}\|^2 \right.\right. \nonumber \\&\left.\left.+\sum_{i=1}^U\frac{p_i^{\max}}{D}2\sigma_i^2\delta^2\right)+\epsilon^2z^2\right)
\nonumber \\&=\left(\frac{\alpha^2L}{2}\Omega_{BEV}-\alpha\omega_{BEV}\right)\|\mathbf{g}_t\|^2
+\frac{\alpha^2L}{2}(\Omega_{BEV}\delta^2+\epsilon^2z^2),
\end{align}
where
\begin{align}\label{TayloromegaBEVs}
\omega_{BEV}&=\sum_{i=1}^M \sqrt{\frac{p_i^{\max}\pi}{2D}}\sigma_i-\sum_{n=1}^N \sqrt{\frac{p_n^{\max}\pi}{2D}}\sigma_n,\\
\Omega_{BEV}&=(U+N)\sum_{i=1}^U\frac{2\sigma_i^2 p_i^{\max}}{D}.\label{TayloromegaBEVL}
\end{align}

If $\mathbb{E}(F(\mathbf{w}_{t})-F(\mathbf{w}_{t-1}))\leq 0$, the objective decreases monotonically, then FL converges in mean. Thus, to ensure the convergence, we have the following convergence condition
\begin{align}\label{conditionbev1}
\frac{\alpha^2L}{2}\Omega_{BEV}-\alpha\omega_{BEV}<0.
\end{align}

Now extend the expectation over randomness in the trajectory, and perform a telescoping sum over the $T$ iterations:
\begin{align}\label{sumexpectationBEV}
F(\mathbf{w}_{0})-F(\mathbf{w}^*)&\geq F(\mathbf{w}_{0})-\mathbb{E}[F(\mathbf{w}_T)]
\nonumber \\
&
=\mathbb{E}\left[\sum_{t=1}^{T}(F(\mathbf{w}_{t-1})-F(\mathbf{w}_{t}))\right]
\nonumber \\
&
\geq\mathbb{E}\left[ \sum_{t=1}^{T}\left(\left(\alpha\omega_{BEV}-\frac{\alpha^2L}{2}\Omega_{BEV}\right)\|\mathbf{g}_t\|^2 \right.\right. \nonumber \\&\left.\left.-\frac{\alpha^2L}{2}(\Omega_{BEV}\delta^2+\epsilon^2z^2)\right)\right].%
\end{align}
We can rearrange this inequality to yield the rate:
\begin{align}\label{sumexpectationBEV1}
\mathbb{E}&\left[\sum_{t=1}^{T}\left(\alpha\omega_{BEV}-\frac{\alpha^2L}{2}\Omega_{BEV}\right)\|\mathbf{g}_t\|^2\right] \nonumber \\&\leq F(\mathbf{w}_{0})-F(\mathbf{w}^*)+\frac{\alpha^2L}{2}\sum_{t=1}^{T}(\Omega_{BEV}\delta^2+\epsilon^2z^2).
\end{align}

If FL converges, $\alpha\omega_{BEV}-\frac{\alpha^2L}{2}\Omega_{BEV}> 0$, and then we get
\begin{align}\label{sumexpectationBEV2}
&\mathbb{E}\left[\sum_{t=1}^{T}\frac{1}{T}\|\mathbf{g}_t\|^2\right]\nonumber\\
&\leq  \frac{F(\mathbf{w}_{0})-F(\mathbf{w}^*) +\frac{\alpha^2L}{2}\sum_{t=1}^{T}(\Omega_{BEV}\delta^2+\epsilon^2z^2)}{T(\alpha\omega_{BEV}-\frac{\alpha^2L}{2}\Omega_{BEV})}.
\end{align}

Let $\alpha=\frac{\omega_{BEV}}{L\Omega_{BEV}\sqrt{T}}\bar{\alpha}$, where $\bar{\alpha}< 2\sqrt{T}$ is a positive constant, and then we have
\begin{align}
&\mathbb{E}\left[\sum_{t=1}^{T}\frac{1}{T}\|\mathbf{g}_t\|^2\right]\nonumber\\
&\leq  \frac{F(\mathbf{w}_{0})-F(\mathbf{w}^*) +\frac{\bar{\alpha}^2\omega_{BEV}^2}{2L\Omega_{BEV}}\left(\delta^2+\frac{1}{\Omega_{BEV}}\epsilon^2 z^2\right)}{T\left(\bar{\alpha}\frac{\omega_{BEV}^2}{L\Omega_{BEV}\sqrt{T}} -\frac{\bar{\alpha}^2\omega_{BEV}^2}{2LT\Omega_{BEV}}\right)}\nonumber\\
&=
\frac{\frac{L\Omega_{BEV}}{\omega_{BEV}^2}(F(\mathbf{w}_{0})-F(\mathbf{w}^*)) +\frac{\bar{\alpha}^2}{2}\left(\delta^2+\frac{1}{\Omega_{BEV}}\epsilon^2 z^2\right)}{T(\frac{\bar{\alpha}}{\sqrt{T}}-\frac{\bar{\alpha}^2}{2T})}
\nonumber\\
&\leq\frac{\frac{L\Omega_{BEV}}{\omega_{BEV}^2}(F(\mathbf{w}_{0})-F(\mathbf{w}^*)) +\frac{\bar{\alpha}^2}{2}\left(\delta^2+\frac{1}{\Omega_{BEV}}\epsilon^2 z^2\right)}{T\frac{\bar{\alpha}}{2\sqrt{T}}}
\nonumber\\
&=\frac{\frac{2L\Omega_{BEV}}{\bar{\alpha}\omega_{BEV}^2}(F(\mathbf{w}_{0})-F(\mathbf{w}^*)) +\bar{\alpha}\left(\delta^2+\frac{1}{\Omega_{BEV}}\epsilon^2 z^2\right)}{\sqrt{T}}.\label{sumexpectationBEV3}
\end{align}
\end{appendices}

\bibliographystyle{IEEEtran}
\bibliography{ref}
\end{document}